\documentclass[10pt,onecolumn,letterpaper]{article}

\usepackage[top=0.7in,bottom=0.7in,left=0.6in,right=0.7in,columnsep=0.35in]{geometry}

\usepackage[utf8]{inputenc} 
\usepackage[T1]{fontenc}    
\usepackage{hyperref}       
\usepackage{url}            
\usepackage{booktabs}       
\usepackage{amsfonts}       
\usepackage{nicefrac}       
\usepackage{microtype}      

\usepackage{graphicx}
\usepackage{subfigure}

\usepackage{hhline}
\usepackage{amsmath,amsthm,amssymb}
\usepackage{color}

\newcommand{\E}{\mathbb{E}}
\newcommand{\Var}{\mathbb{V}{\rm ar}}

\newcommand{\R}{\mathbb{R}}
\newcommand{\leqdel}{\mathop{\leq}_{1-\delta}}

\newtheorem{theorem}{Theorem}
\newtheorem{corollary}{Corollary}
\newtheorem{lemma}{Lemma}
\newtheorem{proposition}{Proposition}
\newtheorem{definition}{Definition}

\usepackage[utf8]{inputenc}
\usepackage{authblk}
\title{Novel Change of Measure Inequalities with Applications to PAC-Bayesian Bounds  and Monte Carlo Estimation}
\author[1]{Yuki Ohnishi}
\author[2]{Jean Honorio}
\affil[1]{Department of Statistics, Purdue University}
\affil[2]{Department of Computer Science, Purdue University}

\begin{document}

\maketitle

\begin{abstract}
We introduce several novel change of measure inequalities for two families of divergences: $f$-divergences and $\alpha$-divergences. We show how the variational representation for $f$-divergences leads to novel change of measure inequalities. We also present a multiplicative change of measure inequality for $\alpha$-divergences and a generalized version of Hammersley-Chapman-Robbins inequality. Finally, we present several applications of our change of measure inequalities, including PAC-Bayesian bounds for various classes of losses and non-asymptotic intervals for Monte Carlo estimates.
\end{abstract}

\section{Introduction}
\label{Introduction}
The Probably Approximate Correct (PAC) Bayesian inequality was introduced by \cite{Shawe-Taylor:1997:PAB:267460.267466} and  \cite{McAllester:1999:PMA:307400.307435}. This framework allows us to produce PAC performance bounds (in the sense of a loss function) for Bayesian-flavored estimators \cite{Guedj2019}, and several extensions have been proposed to date (see e.g., \cite{Seeger:2003:PGE:944919.944929, McAllester2003a, McAllester2003b, Y.Seldin, AmiranAmbroladze}). The core of these theoretical results is summarized by a \textit{change of measure inequality}. The change of measure inequality is an expectation inequality involving two probability measures where the expectation with respect to one measure is upper-bounded by the divergence between the two measures and the moments with respect to the other measure. The change of measure inequality also plays a major role in information theory. For instance, \cite{KATSOULAKIS2017513} derived robust uncertainty quantification bounds for statistical estimators of interest with change of measure inequalities.
	Recent research efforts have been put into more generic perspectives to get PAC-Bayes bounds and to get rid of assumptions such as boundedness of the loss function (e.g., \cite{Lever2013, NIPS2013_4903, NIPS2019_9387, NIPS2019_8539, pmlr-v98-grunwald19a}). All PAC-Bayesian bounds contained in these works massively rely on one of the most famous change of measure inequalities, the Donsker-Varadhan representation for the KL-divergence.
	Several change of measure inequalities had been proposed along with PAC-Bayes bounds lately.  \cite{LucBegin} proposed a proof scheme of PAC-Bayesian bounds based on the R\'enyi divergence.  \cite{honorio} proposed an inequality for the $\chi^2$ divergence and derived a PAC-Bayesian bound for linear classification.  \cite{Alquier2018} proposed a novel change of measure inequality and PAC-Bayesian bounds based on the $\alpha$-divergence.
The aforementioned works were proposed for specific purposes. A comprehensive study on change of measure inequalities has not been performed yet. Our work proposes several novel and general change of measure inequalities for two families of divergences: $f$-divergences and $\alpha$-divergences.  It is a well-known fact that the $f$-divergence can be variationally characterized as the maximum of an optimization problem rooted in the convexity of the function $f$. This variational representation has been recently used in various applications of information theory, such as $f$-divergence estimation \cite{Nguyen:2010:EDF:1921943.1921980} and quantification of the bias in adaptive data analysis \cite{Jiao}. Recently, \cite{Ruderman:2012} showed that the variational representation of the $f$-divergence can be tightened when constrained to the space of probability densities. 

Our main contributions are as follows:
\begin{itemize}
	\item By using the variational representation for $f$-divergences, we derive several change of measure inequalities. We perform the analysis for the constrained regime (to the space of probability densities) as well as the unconstrained regime.
	\item We present a multiplicative change of measure inequality for the family of $\alpha$-divergences. This generalizes the previous results \cite{Alquier2018, honorio} for the $\alpha$-divergence, which in turns apply to PAC-Bayes inequalities for types of losses not considered before.
	\item We also generalize prior results for the Hammersley-Chapman-Robbins inequality \cite{lehmann1998} from the particular $\chi^2$ divergence, to the family of $\alpha$-divergences.
	\item We provide new PAC-Bayesian bounds with the $\alpha$-divergence and the $\chi^2$-divergence from our novel change of measure inequalities for bounded, sub-Gaussian, sub-exponential and bounded-variance loss functions. Our results are either novel, or have a tighter complexity term than existing results in the literature, and pertain to important machine learning prediction problems, such as regression, classification and structured prediction.
	\item We provide a new scheme for estimation of non-asymptotic intervals for Monte Carlo estimates. Our results indicate that the empirical mean over a sampling distribution concentrates around an expectation with respect to any arbitrary distribution.
\end{itemize}

\section{Change of Measure Inequalities}
\label{ChangeofMeasureInequalities} 

\begin{table*}[t]
	\centering
	\caption{Summary of the change of measure inequalities. For simplicity, we denote $\E_P[\cdot] \equiv \E_{h \sim P}[\cdot]$ and $\phi \equiv \phi(h)$.}
	\label{tab:1}
	\begin{tabular}{ p{2cm}p{2.5cm}p{7.7cm}p{1.5cm}}
		
		\textbf{Bound Type}& \textbf{Divergence} & \textbf{Uppper-Bound for Every $Q$ and a Fixed $P$} & \textbf{Reference}\\
		\hhline{====}
		Constrained & KL   & $\E_Q[\phi]\leq  KL(Q\|P)+\log(\E_P[e^{\phi}])$ &   \cite{McAllester:1999:PMA:307400.307435}\\
		Variational & Pearson $\chi^2$   & $\E_Q[\phi] \leq \chi^2(Q\|P)+\E_P[\phi]+\frac{1}{4}
		\Var_P[\phi]$ & Lemma \ref{lemma_cmi_constrainedchi2}  \\

		Representation & Total Variation   &  $\E_Q[\phi] \leq TV(Q\|P)+\E_P[\phi]$ for $\phi\in[0,1]$ & Lemma \ref{lemma_cmi_tv}  \\
		
		\hline
		Unconstrained & KL   &  $\E_Q[\phi]\leq  KL(Q\|P)+(\E_P[e^{\phi}]-1)$ & \cite{McAllester:1999:PMA:307400.307435}  \\

		Variational & Pearson $\chi^2$   & $\E_Q[\phi] \leq \chi^2(Q\|P)+\E_P[\phi]+\frac{1}{4}
		\E_P[\phi^2]$ & Lemma \ref{lemma_cmi_unconstrainchi2}  \\

		Representation & Total Variation  & $\E_Q[\phi] \leq TV(Q\|P)+\E_P[\phi]$ for $\phi\in[0,1]$  &  \\

		&$\alpha$   & $\E_Q[\phi] \leq D_{\alpha}(Q\|P)+\frac{(\alpha-1)^{\frac{\alpha}{\alpha-1}}}{\alpha}\E_P[\phi^{\frac{\alpha}{\alpha-1}}] + \frac{1}{\alpha(\alpha-1)}$ & Lemma \ref{lemma_cmi_unconstrainedalpha} \\

		& Squared Hellinger  & $ \E_Q[\phi]\leq H^2(Q\|P) + \E_P[\frac{\phi}{1-\phi}] $ for $\phi\ < 1$ & Lemma \ref{lemma_cmi_unconstrainedhel} \\ 
		
		& Reverse KL  &$\E_Q[\phi]\leq \overline{KL}(Q\|P) + \E_P[\log(\frac{1}{1-\phi})]$ for $\phi\ < 1$& Lemma \ref{lemma_cmi_unconstrainedrevkl} \\  
		
		&Neyman $\chi^2$  & $\E_Q[\phi]\leq \overline{\chi^2}(Q\|P) + 2 - 2\E_P[\sqrt{1-\phi}]$ for $\phi\ < 1$ & Lemma \ref{lemma_cmi_unconstrainedrevchi2} \\  
		
		\hline
		Multiplicative & Pearson $\chi^2$  & $\E_Q[\phi]\leq \sqrt{(\chi^2(Q\|P)+1)\E_P[\phi^2]}$ & \cite{honorio} \\ 
		
		& $\alpha$  &  $\E_Q[\phi]\leq (\alpha(\alpha-1)D_{\alpha}(Q\|P)+1)^{\frac{1}{\alpha}}(\E_P[|\phi|^{\frac{\alpha}{\alpha-1}}])^{\frac{\alpha-1}{\alpha}}$ & \cite{Alquier2018} \\ 
		
		\hline
		Generalized&Pearson $\chi^2$  &  $ \chi^2(Q\|P) \geq \frac{(\E_Q[\phi] - \E_P[\phi])^2}{\Var_P[\phi]}$& \cite{lehmann1998} \\ 
		
		HCR  &Pseudo $\alpha$  &  $ |\E_Q[\phi]-\E_P[\phi]|\leq \widetilde{\mathcal{D_{\alpha}}}(Q\|P)^{\frac{1}{\alpha}}(\E_P[|\phi-\mu_P|^{\frac{\alpha}{\alpha-1}}])^{\frac{\alpha-1}{\alpha}}$ & Lemma \ref{lemma_generalhcr} \\ 
		
	\end{tabular}
\end{table*}
\begin{table*}[t]
	\caption{Some common $f$-divergences with corresponding generator.}
	\centering
	\label{tab:div}
	\begin{tabular}{p{3cm}p{7cm}p{2.3cm}}
		
		\textbf{Divergence}& \textbf{Formula with probability measures $P$ and $Q$ defined on a common space $\mathcal{H}$} & \textbf{Corresponding Generator$f(t)$}\\
		\hhline{===}
		KL   & $ KL(Q\|P)=\int_{\mathcal{H}}\log\frac{dQ}{dP}dQ$ &$t\log t-t +1$\\
		
		Reverse KL &$\overline{KL}(Q\|P)=\int_{\mathcal{H}}\log\frac{dP}{dQ}dP$ & $-\log t$\\

		Pearson $\chi^2$   & $ \chi^2(Q\|P)=\int_{\mathcal{H}}(\frac{dQ}{dP}-1)^2dP$  & $(t-1)^2$ \\
		
		Neyman $\chi^2$  & $\overline{\chi^2}(Q\|P)=\int_{\mathcal{H}}(\frac{dP}{dQ}-1)^2dQ$  & $\frac{(1-t)^2}{t}$\\  
		
		Total Variation   &  $TV(Q\|P)=\frac{1}{2}\int_{\mathcal{H}}|\frac{dQ}{dP}-1|dP$ & $\frac{1}{2}|t-1|$ \\
		
		Squared Hellinger  & $H^2(Q\|P) = \int_{\mathcal{H}}(\sqrt{\frac{dQ}{dP}}-1)^2dP$ & $(\sqrt{t}-1)^2$ \\ 
		
		$\alpha$   & $ D_{\alpha}(Q\|P)=\frac{1}{\alpha(\alpha-1)}\int_{\mathcal{H}}(|\frac{dQ}{dP}|^{\alpha}-1)dP$ & $\frac{t^{\alpha}-1}{\alpha(\alpha-1)}$ \\ 
		
		Pseudo $\alpha$   & $ \widetilde{\mathcal{D_{\alpha}}}(Q\|P)=\int_{\mathcal{H}}|\frac{dQ}{dP}-1|^{\alpha}dP$ & $|t-1|^{\alpha}$ \\ 
		
		$\phi_p$ \cite{Alquier2018}   & $ D_{\phi_p -1}(Q\|P)=\int_{\mathcal{H}}(|\frac{dQ}{dP}|^{p}-1)dP$ & $t^{p}-1$ \\ 
	\end{tabular}
\end{table*}

In this section, we formalize the definition of $f$-divergences and present the constrained representation (to the space of probability measures) as well as the unconstrained representation. Then, we provide different change of measure inequalities for several divergences. We also provide multiplicative bounds as well as a generalized Hammersley-Chapman-Robbins bound. Table \ref{tab:1} summarizes our results.
\subsection{Change of Measure Inequality from the Variational Representation of $f$-divergences}
Let $f : (0, +\infty)  \rightarrow \R$ be a convex function. The convex conjugate $f^*$ of $f$ is defined by:
\begin{equation}
\label{eq_conjugate}
f^* (y) = \sup_{x\in \R} (xy - f(x)).
\end{equation}
The definition of $f^{*}$ yields the following Young-Fenchel inequality
$$f(x) \geq xy - f^{*} (y)$$
which holds for any $y$. Using the notation of convex conjugates, the $f$-divergence and its variational representation is defined as follows.
\theoremstyle{definition}
\begin{definition}[$f$-divergence and its variational representation]
	\label{def_fdiv}
	Let $\mathcal{H}$ be any arbitrary domain. Let $P$ and $Q$ denote the probability measures over the Borel $\sigma$-field on $\mathcal{H}$. Additionally, let $f$: $[0, \infty) \rightarrow \mathbb{R}$ be a convex and lower semi-continuous function that satisfies $f(1)= 0$. 
	\begin{equation*}
	D_f(Q\|P):=\E_P\bigg[f\bigg(\frac{dQ}{dP}\bigg)\bigg]
	\end{equation*}
	
\end{definition} 
For simplicity, we denote $\E_P[\cdot] \equiv \E_{h \sim P}[\cdot]$ in the sequel.
Many common divergences, such as the KL-divergence and the Hellinger divergence, are members of the family of $f$-divergences, coinciding with a particular choice of $f(t)$. Table \ref{tab:div} presents the definition of each divergence with the corresponding generator $f(t)$.
It is well known that the $f-$divergence can be characterized as the following variational representation.
\begin{lemma}[Variational representation of $f-$divergence, Lemma 1 in \cite{Nguyen_MJ2010}]
	\label{lemma_var_rep_f}
	Let $\mathcal{H}$, $P$, $Q$ and $f$ be defined as in Definition \ref{def_fdiv}. Let $\phi$: $\mathcal{H}\rightarrow\R$ be a real-valued function. The $f$-divergence from $P$ to $Q$ is characterized as
	\begin{equation*}
	D_f(Q\|P)\geq \sup_{\phi} \E_Q[\phi]-\E_P[f^*(\phi)]
	\end{equation*}
\end{lemma}
\cite{Ruderman:2012} shows that this variational representation for $f$-divergences can be tightened. 
\begin{theorem} [Change of measure inequality from the constrained variational representation for $f$-divergences \cite{Ruderman:2012}] 
	\label{theorem_cmi_fdiv}
	Let $\mathcal{H}$, $P$, $Q$ and $f$ be defined as in Definition \ref{def_fdiv}. Let $\phi$: $\mathcal{H}\rightarrow\R$ be a real-valued function. Let $\Delta(\mu):=\{g: \mathcal{H}  \rightarrow \mathbb{R}:g \geq 0, \|g\|_1 = 1\}$ denote the space of probability densities with respect to $\mu$, where the norm is defined as  $\|g\|_1 := \int_{\mathcal{H}} |g| d\mu$, given a measure $\mu$ over $\mathcal{H}$. The general form of the change of measure inequality for f-divergences is given by
	\begin{equation*}
	\begin{split}
	\E_Q[\phi] & \leq D_f(Q\|P) + (\mathbb{I}^{R}_{f, P})^*(\phi)\\
	(\mathbb{I}^{R}_{f, P})^*(\phi) & =\sup_{p\in \Delta(P)}{\E_P[\phi p]-\E_{P}[f(p)]}
	\end{split}
	\end{equation*}
	where p is constrained to be a probability density function. 
\end{theorem}
The famous Donsker-Varadhan representation for the KL-divergence, which is used in  most PAC-Bayesian bounds, can be actually derived from this tighter representation by setting $f(t)=t\log(t)$. However, it is not always easy to find a closed-form solution for Theorem \ref{theorem_cmi_fdiv}, as it requires to resort to variational calculus, and in some cases, there is no closed-form solution. In such a case, we can use the following corollary to obtain looser bounds, but only requires to find a convex conjugate.
\begin{corollary}[Change of measure inequality from the unconstrained variational representation for $f$-divergences]
	\label{cor_cmi_unconstrainedfdiv}
	Let $P$, $Q$, $f$ and $\phi$ be defined as in Theorem \ref{theorem_cmi_fdiv}. By Definition \ref{def_fdiv}, we have
	\begin{equation*}
	\begin{split}
	\forall Q \text{ on } \mathcal{H}:\  \E_Q[\phi]\leq D_f(Q\|P) + \E_P[f^*(\phi)]
	\end{split}
	\end{equation*}
\end{corollary}
Detailed proofs can be found in Appendix \ref{appendix_sec2}. By choosing a right function $f$ and deriving the constrained maximization term $(\mathbb{I}^{R}_{f, P})^*(\phi)$ with the help of variational calculus, we can create an upper-bound based on the corresponding divergence $D_f(Q\|P)$. Next, we discuss the case of the $\chi^2$ divergence.
\begin{lemma}[Change of measure inequality from the constrained representation of the $\chi^2$-divergence] 
	\label{lemma_cmi_constrainedchi2}
	Let $P$, $Q$ and $\phi$ be defined as in Theorem \ref{theorem_cmi_fdiv}, we have
	\begin{equation*}
	\forall Q \text{ on } \mathcal{H}:\ \E_Q[\phi] \leq \chi^2(Q\|P)+\E_P[\phi]+\frac{1}{4}
	\Var_P[\phi]
	\end{equation*}
\end{lemma}

The bound in Lemma \ref{lemma_cmi_constrainedchi2} is slightly tighter than the one without the constraint. The change of measure inequality without the constraint is given as follows.
\begin{lemma}[Change of measure inequality from the unconstrained representation of the Pearson $\chi^2$-divergence]
	\label{lemma_cmi_unconstrainchi2}
	Let $P$, $Q$ and $\phi$ be defined as in Theorem \ref{theorem_cmi_fdiv}, we have
	\begin{equation*}
	\forall Q \text{ on } \mathcal{H}:\  \E_Q[\phi] \leq \chi^2(Q\|P)+\E_P[\phi]+\frac{1}{4}\E_P[\phi^2]
	\end{equation*}
\end{lemma}

As might be apparent, the bound in Lemma \ref{lemma_cmi_constrainedchi2} is tighter than the one in Lemma \ref{lemma_cmi_unconstrainchi2} by $(\E_P[\phi])^2$ because $\Var_P[\phi] \leq \E_P[\phi^2]$. Next, we discuss the case of the total variation divergence.
\begin{lemma}[Change of measure inequality from the constrained representation of the total variation divergence]
	\label{lemma_cmi_tv}
	Let $\phi$: $\mathcal{H}\rightarrow [0,1]$ be a real-valued function.
	Let $P$ and $Q$  be defined as in Theorem \ref{theorem_cmi_fdiv}, we have
	\begin{equation*}
	\begin{split}
	\forall Q \text{ on } \mathcal{H}:\  \E_Q[\phi] \leq TV(Q\|P)+\E_P[\phi]
	\end{split}
	\end{equation*}
\end{lemma}
Interestingly, we can obtain the same bound on the total variation divergence even if we use the unconstrained variational representation. Next, we state our result for $\alpha$-divergences.
\begin{lemma} [Change of measure inequality from the unconstrained representation of the $\alpha$-divergence]
	\label{lemma_cmi_unconstrainedalpha}
	Let $P$, $Q$ and $\phi$ be defined as in Theorem \ref{theorem_cmi_fdiv}. For $\alpha>1$, we have
	\begin{equation*}
	\begin{split}
	\forall Q \text{ on } \mathcal{H}:\  \E_Q[\phi] \leq D_{\alpha}(Q\|P)+\frac{(\alpha-1)^{\frac{\alpha}{\alpha-1}}}{\alpha}\E_P[\phi^{\frac{\alpha}{\alpha-1}}] + \frac{1}{\alpha(\alpha-1)}
	\end{split}
	\end{equation*}
\end{lemma}
We can obtain the bounds based on the squared Hellinger divergence $H^2(Q\|P)$, the reverse KL-divergence $\overline{KL}(Q\|P)$ and  the Neyman $\chi^2$-divergence  $\overline{\chi^2}(Q\|P)$ in a similar fashion. 
\begin{lemma}[Change of measure inequality from the unconstrained representation of the squared Hellinger divergence]
	\label{lemma_cmi_unconstrainedhel}
	Let $\phi$: $\mathcal{H}\rightarrow (-\infty,1)$ be a real-valued function.
	Let $P$ and $Q$  be defined as in Theorem \ref{theorem_cmi_fdiv}, we have
	\begin{equation*}
	\begin{split}
	\forall Q \text{ on } \mathcal{H}:\  \E_Q[\phi]\leq H^2(Q\|P) + \E_P\bigg[\frac{\phi}{1-\phi}\bigg] 
	\end{split}
	\end{equation*}
\end{lemma}
Similarly, we obtain the following bound for the reverse-KL divergence.
\begin{lemma}[Change of measure inequality from the unconstrained representation of the reverse KL-divergence]
	\label{lemma_cmi_unconstrainedrevkl}
	Let $\phi$: $\mathcal{H}\rightarrow (-\infty,1)$ be a real-valued function.
	Let $P$ and $Q$  be defined as in Theorem \ref{theorem_cmi_fdiv}, we have
	\begin{equation*}
	\begin{split}
	\forall Q \text{ on } \mathcal{H}:\  \E_Q[\phi]\leq \overline{KL}(Q\|P) + \E_P\bigg[\log\bigg(\frac{1}{1-\phi}\bigg)\bigg]
	\end{split}
	\end{equation*}
\end{lemma}
Finally, we prove our result for the Neyman $\chi^2$ divergence based on a similar approach. 
\begin{lemma} [Change of measure inequality from the unconstrained representation of the Neyman $\chi^2$-divergence]
	\label{lemma_cmi_unconstrainedrevchi2}
	Let $\phi$: $\mathcal{H}\rightarrow (-\infty,1)$ be a real-valued function.
	Let $P$ and $Q$  be defined as in Theorem \ref{theorem_cmi_fdiv}, we have
	\begin{equation*}
	\begin{split}
	\forall Q \text{ on } \mathcal{H}:\  \E_Q[\phi]\leq \overline{\chi^2}(Q\|P) + 2 - 2\E_P[\sqrt{1-\phi}]
	\end{split}
	\end{equation*}
\end{lemma}
\subsection{Multiplicative Change of Measure Inequality for $\alpha$-divergences}
First, we state a known result for the $\chi^2$ divergence.
\begin{lemma}[Multiplicative change of measure inequality for the $\chi^2$-divergence \cite{honorio}] 
	\label{honorio_bound}
	Let $P$, $Q$ and $\phi$ be defined as in Theorem \ref{theorem_cmi_fdiv}, we have
	\begin{equation*}
	\begin{split}
	\forall Q \text{ on } \mathcal{H}:\  \E_Q[\phi]\leq \sqrt{(\chi^2(Q\|P)+1)\E_P[\phi^2]}
	\end{split}
	\end{equation*}
\end{lemma}

First, we note that the $\chi^2$ divergence is an $\alpha$-divergence for $\alpha=2$. Next, we generalize the above bound for any $\alpha$-divergence.
\begin{lemma}[Multiplicative change of measure inequality for the $\alpha$-divergence]
	\label{lemma_mulalpha}
	Let $P$, $Q$ and $\phi$ be defined as in Theorem \ref{theorem_cmi_fdiv}. For any $\alpha>1$, we have
	\begin{equation*}
	\begin{split}
	\forall Q \text{ on } \mathcal{H}:\  \E_Q[\phi]\leq \big(\alpha(\alpha-1)D_{\alpha}(Q\|P)+1\big)^{\frac{1}{\alpha}}\big(\E_P[|\phi|^{\frac{\alpha}{\alpha-1}}]\big)^{\frac{\alpha-1}{\alpha}}
	\end{split}
	\end{equation*}
\end{lemma}
Our bound is stated in the form of $\alpha$-divergence. By choosing $\alpha=2$, we have the same bound as Lemma \ref{honorio_bound} where $\chi^2(Q\|P) = 2D_{\alpha}(Q\|P)$. We will later apply the above $\alpha$-divergence change of measure to obtain PAC-Bayes  inequalities for types of losses not considered before \cite{Alquier2018, honorio}.

\subsection{A Generalized Hammersley-Chapman-Robbins (HCR) Inequality}
The HCR inequality is a famous information theoretic inequality for the $\chi^2$-divergence.
\begin{lemma}[HCR inequality \cite{lehmann1998}] 
	\label{lemma_hcr}
	Let $P$, $Q$ and $\phi$ be defined as in Theorem \ref{theorem_cmi_fdiv}, we have
	\begin{equation*}
	\begin{split}
	\forall Q \text{ on } \mathcal{H}:\  \chi^2(Q\|P) \geq \frac{(E_Q[\phi] - \E_P[\phi])^2}{\Var_P[\phi]}
	\end{split}
	\end{equation*}
\end{lemma}
Next, we generalize the above bound for $\alpha$-divergence.
\begin{lemma}[The generalization of HCR inequality.] Let $P$, $Q$ and $\phi$ be defined as in Theorem \ref{theorem_cmi_fdiv}. For any $\alpha>1$, we have
	\label{lemma_generalhcr}
	\begin{equation*}
	\begin{split}
	\forall Q \text{ on } \mathcal{H}:\  \big|\E_Q[\phi]-\E_P[\phi]\big| \leq
	\widetilde{\mathcal{D_{\alpha}}}(Q\|P)^{\frac{1}{\alpha}}\big(\E_P[|\phi-\mu_P|^{\frac{\alpha}{\alpha-1}}]\big)^{\frac{\alpha-1}{\alpha}}
	\end{split}
	\end{equation*}
	where $\widetilde{\mathcal{D_{\alpha}}}(Q\|P)=\int_{\mathcal{H}}|\frac{dQ}{dP}-1|^{\alpha}dP$ and $\mu_P = \E_P[\phi]$.
\end{lemma}
We call $\widetilde{\mathcal{D_{\alpha}}}(Q\|P)$ a pseudo $\alpha$-divergence, which is a 
member of the family of $f$-divergences with $f(t)=|t-1|^{\alpha}$ for $\alpha > 1$. See Appendix \ref{appendix_alpha_is_f} for the formal proof.
Straightforwardly, we can obtain Lemma \ref{lemma_hcr} by choosing $\alpha=2$. 

\section{Application to PAC-Bayesian Theory}
\label{pacbayesianbounds}

\begin{table*}[t]
	\centering
	\caption{Summary of PAC-Bayesian bounds with $\alpha$-divergence and $\chi^2$-divergence.}
	\label{tab:3}
	\begin{tabular}{ p{2.3cm}p{0.3cm}p{8cm}p{3.5cm}}
		
		\multicolumn{2}{l}{\textbf{Loss \& Divergence}} & \textbf{Generalization upper bound for $R_D(G_Q)-R_S(G_Q)$, for every $Q$, and  a fixed $P$} & \textbf{Reference}\\
		\hhline{====}
		Bounded Loss& $\alpha$ & $O\big(R[\frac{1}{m}\log(\frac{1}{\delta})]^{\frac{1}{2}}D_{\alpha}(Q\|P)^{\frac{1}{2\alpha}}\big)$ & Proposition \ref{prop_boundloss}\\
		
		& $\alpha$   & $O\big((\frac{1}{m})^{\frac{1}{2}} [D_{\alpha}(Q\|P) + (R^2\log(\frac{1}{\delta}))^{\frac{\alpha}{\alpha-1}} ]^{\frac{1}{2}}\big)$ & Proposition \ref{prop_boundloss}  \\
		
		& $\chi^2$ & $O\big( R [\frac{1}{m}\log(\frac{1}{\delta})]^{\frac{1}{2}}\chi^2(Q\|P)^{\frac{1}{4}}  \big)$ &  Corollary \ref{cor_pac_chi_bounded_loss}\\
		
		& $\chi^2$   & $O\big( (\frac{1}{m})^{\frac{1}{2}}[\chi^2(Q\|P) + (R^2\log(\frac{1}{\delta}))^2]^{\frac{1}{2}}\big)$ & Corollary \ref{cor_pac_chi_bounded_loss}  \\
		
		0-1 loss& $\chi^2$ & $O\big((\frac{1}{m\delta})^{\frac{1}{2}}\chi^2(Q\|P)^{\frac{1}{2}}\big)$ &  \cite{honorio, LucBegin}\\
		
		\hline
		Sub-Gaussian& $\alpha$ & $O\big(\sigma[\frac{1}{m}\log(\frac{1}{\delta})]^{\frac{1}{2}}D_{\alpha}(Q\|P)^{\frac{1}{2\alpha}}\big)$ & Proposition \ref{prop_pac_subgauss}\\
		
		& $\alpha$   & $O\big((\frac{1}{m})^{\frac{1}{2}}[D_{\alpha}(Q\|P) + (\sigma^2\log(\frac{1}{\delta}))^{\frac{\alpha}{\alpha-1}}]^{\frac{1}{2}}\big)$ & Proposition \ref{prop_pac_subgauss}  \\
		
		& $\alpha$   & $O\big(\sigma(\frac{1}{m})^{\frac{1}{2}}(\frac{1}{\delta})^{\frac{\alpha-1}{\alpha}}  D_{\alpha}(Q\|P)^{\frac{1}{\alpha}}\big)$ &  Theorem 1 \& Proposition 6 in \cite{Alquier2018}  \\
		
		& $\chi^2$ & $ O\big(\sigma[\frac{1}{m}\log(\frac{1}{\delta})]^{\frac{1}{2}}\chi^2(Q\|P)^{\frac{1}{4}}\big)$ &  Corollary \ref{cor_pac_subgauss_chi2}\\
		
		& $\chi^2$   & $ O\big((\frac{1}{m})^{\frac{1}{2}}[\chi^2(Q\|P)+(\sigma^2\log(\frac{1}{\delta}))^2]^{\frac{1}{2}}\big)$ & Corollary \ref{cor_pac_subgauss_chi2}  \\
		
		& $\chi^2$ & $O\big( \sigma(\frac{1}{m\delta})^{\frac{1}{2}}\chi^2(Q\|P)^{\frac{1}{2}}\big)$ &  Theorem 1 \& Proposition 6 in \cite{Alquier2018}\\
		
		\hline
		Sub-exponential& $\alpha$ & $O\big(\frac{\beta}{m}\log( \frac{1}{\delta})D_{\alpha}(Q\|P)^{\frac{1}{2\alpha}}\big)$ & Proposition \ref{prop_pac_subexp}\\
		
		For & $\alpha$   & $O\big((\frac{1}{m})^{\frac{1}{2}}[D_{\alpha}(Q\|P)+m^{\frac{-\alpha}{\alpha-1}}(\beta\log( \frac{1}{\delta}))^{\frac{2\alpha}{\alpha-1}}]^{\frac{1}{2}}\big)$ & Proposition \ref{prop_pac_subexp}  \\
		
		$m<\frac{2\beta^2\log(\frac{2}{\delta})}{\sigma^2}$& $\chi^2$ & $O\big(\frac{\beta}{m}\log( \frac{1}{\delta})\chi^2(Q\|P)^{\frac{1}{4}}\big)$ &  Corollary \ref{cor_pac_sub_exp_chi}\\
		
		& $\chi^2$   & $ O\big( (\frac{1}{m})^{\frac{1}{2}} [(\chi^2(Q\|P)+\frac{1}{m^2}(\beta\log( \frac{1}{\delta}))^4)]^{\frac{1}{2}} \big)$ & Corollary \ref{cor_pac_sub_exp_chi}  \\
		
		\hline
		Bounded & $\alpha$ & $O\big( \sigma(\frac{1}{m\delta})^{\frac{1}{2}}D_{\alpha}(Q\|P)^{\frac{1}{2\alpha}} \big)$ & Proposition \ref{prop_pac_boundedvar}\\
		
		Variance& $\alpha$   & $O\big((\frac{1}{m})^{\frac{1}{2}}[D_{\alpha}(Q\|P)+(\frac{\sigma^2}{\delta})^{\frac{\alpha}{\alpha-1}}]^{\frac{1}{2}} \big)$ & Proposition \ref{prop_pac_boundedvar}  \\
		
		& $\alpha$   & $O\big( \sigma(\frac{1}{m})^{\frac{1}{2}} (\frac{D_{\alpha}(Q\|P)}{\delta^{\alpha-1}})^{\frac{1}{\alpha}} \big)$ &   Proposition 4 in \cite{Alquier2018}  \\
		
		& $\chi^2$ & $O\big( \sigma(\frac{1}{m\delta})^{\frac{1}{2}}\chi^2(Q\|P)^{\frac{1}{4}} \big)$ &  Corollary \ref{cor_pac_boundedvar_chi}\\
		
		& $\chi^2$   & $ O\big( (\frac{1}{m})^{\frac{1}{2}}[\chi^2(Q\|P)+(\frac{\sigma^2}{\delta})^2]^{\frac{1}{2}} \big)$ & Corollary \ref{cor_pac_boundedvar_chi}  \\
		
		& $\chi^2$ & $O\big( \sigma(\frac{1}{m\delta})^{\frac{1}{2}}\chi^2(Q\|P)^{\frac{1}{2}} \big)$ &  Corollary 1 in \cite{Alquier2018}\\
		
	\end{tabular}
\end{table*}

In this section, we will explore the applications of our change of measure inequalities. We consider an arbitrary input space \( \mathcal{X}\) and a output space  $\mathcal{Y}$ . The samples \((x,y) \in \mathcal{X} \times \mathcal{Y}\) are input-output pairs. Each example \(\mathnormal{(x,y)}\) is drawn \textit{i.i.d.} according to
a fixed, but unknown, distribution $D$ on $ \mathcal{X} \times \mathcal{Y}$. Let $\ell:\mathcal{Y} \times \mathcal{Y} \to \R$ denote a generic loss function. 
The risk $R_D(h)$ of any predictor \(\mathnormal{h}: \mathcal{X} \rightarrow \mathcal{Y}\) is defined as
the expected loss induced by samples drawn
according to $D$. Given a training set $S$ of $m$ samples,
the empirical risk $R_S(h)$ of any predictor $h$ is defined
by the empirical average of the loss. That is
$$ R_D(h) = \mathop{\E}_{(x,y) \sim D} \ell(h(x), y), \quad  R_S(h) = \frac{1}{|S|} \sum_{(x,y) \in S} \ell(h(x), y)$$

In the PAC-Bayesian framework, we consider a hypothesis space $\mathcal{H}$ of predictors, a prior distribution
$P$ on $\mathcal{H}$, and a posterior distribution $Q$ on $\mathcal{H}$. The
prior is specified before exploiting the information contained in $S$, while the posterior is obtained by running                                                              non
a learning algorithm on $S$. The PAC-Bayesian theory usually studies the stochastic Gibbs predictor $G_Q$.
Given a distribution $Q$ on $\mathcal{H}$, $G_Q$ predicts an example $x$ by drawing a predictor $h$ according to $Q$, and returning $h(x)$. The risk of $G_Q$ is then defined as follows.
For any probability distribution $Q$ on
a set of predictors, the Gibbs risk $R_D(G_Q)$ is the expected
risk of the Gibbs predictor $G_Q$ relative to $D$. Hence,
\begin{equation}
\label{gibbsrisk}
R_D(G_Q) = \mathop{\mathbb{E}}_{(x,y) \sim D}\mathop{\mathbb{E}}_{h \sim Q} \ell(h(x), y)
\end{equation}
Usual PAC-Bayesian bounds give guarantees on the generalization risk $R_D(G_Q)$. Typically, these bounds rely on the empirical risk $R_S(G_Q)$ defined as follows.
\begin{equation}
\label{gibbsempiricalrisk}
R_S(G_Q) = \frac{1}{|S|} \sum_{(x,y) \in S} \mathop{\mathbb{E}}_{h \sim Q} \ell(h(x), y)
\end{equation}

Due to space constraints,  we fully present PAC-Bayes generalization bounds for losses with bounded variance. Other results for bounded losses, sub-Gaussian losses and sub-exponential losses are included in Appendix \ref{appendix_pac_bayes}. Still, we briefly discuss our new results in Section \ref{discussion_pacbayes}.
	
\subsection{Loss Function with Bounded Variance}
\label{section_loss_boundedvar}
In this section, we present our PAC-Bayesian bounds for the loss functions with bounded variance.  Assuming any arbitrary distribution with bounded variance on the loss function $\ell$ (i.e., $ \mathop{\Var}_{(x,y)\sim D}[\ell(h(x),y)]\leq\sigma^2$ for any $h \in \mathcal{H}$), we have the following PAC-Bayesian bounds.

\begin{proposition}[The PAC-Bayesian bounds for loss function with bounded variance]
	\label{prop_pac_boundedvar}
	Let $P$ be a fixed prior distribution over a hypothesis space any $h \in \mathcal{H}$. For a given posterior distribution $Q$ over an infinite hypothesis space $\mathcal{H}$, let $R_D(G_Q)$ and $R_S(G_Q)$ be the Gibbs risk and the empirical Gibbs risk as in Equation \eqref{gibbsrisk} and \eqref{gibbsempiricalrisk} respectively. 
	For the sample size $m$ and $\alpha > 1$, with probability at least $1-\delta$, simultaneously for all posterior distributions $Q$, we have
	\begin{equation*}
	\begin{split}
	R_D(G_Q)  \leq R_S(G_Q)+\sqrt{\frac{\sigma^2}{m\delta}\big(\alpha(\alpha-1)D_{\alpha}(Q\|P)+1\big)^{\frac{1}{\alpha}}}
	\end{split}
	\end{equation*}
	\begin{equation}
	\label{eq_additive_boundedvar}
	\begin{split}
	R_D(G_Q)  \leq R_S(G_Q)+\sqrt{\frac{1}{m}\bigg(D_{\alpha}(Q\|P)+\frac{1}{\alpha(\alpha-1)}\bigg)+\frac{1}{m\alpha}\bigg(\frac{\sigma^2(\alpha-1)}{\delta}\bigg)^{\frac{\alpha}{\alpha-1}}}
	\end{split}
	\end{equation}
\end{proposition}
By setting $\alpha=2$ in Proposition \ref{cor_pac_boundedvar_chi}, we have the following claim.
\begin{corollary}[The PAC-Bayesian bounds with $\chi^2$-divergence for bounded variance loss function]
	\label{cor_pac_boundedvar_chi}
	Let $P$ be any prior distribution over an infinite hypothesis space $\mathcal{H}$. For a given posterior distribution $Q$ over an infinite hypothesis space $\mathcal{H}$, let $R_D(G_Q)$ and $R_S(G_Q)$ be the Gibbs risk and the empirical Gibbs risk as in Equation \eqref{gibbsrisk} and \eqref{gibbsempiricalrisk} respectively. 
	For the sample size $m>0$, with probability at least $1-\delta$, simultaneously for all posterior distributions $Q$, we have
	\begin{equation}
	\label{cor_boundedvar_loss_1}
	\begin{split}
	R_D(G_Q) &  \leq  R_S(G_Q) + \sqrt{\frac{\sigma^2}{m\delta}\sqrt{\chi^2(Q\|P)+1}}
	\end{split}
	\end{equation}	
	\begin{equation*}
	\begin{split}
	R_D(G_Q)   \leq  R_S(G_Q) + \sqrt{\frac{1}{2m}\bigg(\chi^2(Q\|P)+1+\bigg(\frac{\sigma^2}{\delta}\bigg)^2\bigg)}
	\end{split}
	\end{equation*}
\end{corollary}

Now, note that choosing $\Delta(q,p)=|q-p|$ for Equation \eqref{cor_boundedvar_loss_1} results in 
\begin{equation}
\label{eq_cor_1_1_alq}
\begin{split}
R_D(G_Q) &  \leq  R_S(G_Q) + \sqrt{\frac{\sigma^2}{m\delta}(\chi^2(Q\|P)+1)}
\end{split}
\end{equation}	
which was shown in Corollary 1 of \cite{Alquier2018}. We can easily see that \eqref{cor_boundedvar_loss_1} is tighter than \eqref{eq_cor_1_1_alq} due to the choice of $\Delta(q, p)=(q-p)^2$.  Please see Table \ref{tab:3} for details.
\subsection{Discussion}
\label{discussion_pacbayes}
Table \ref{tab:3} presents various PAC-Bayesian bounds based on our change of measure inequalities depending on different assumptions on the loss function $\ell$, and compares them with existing results in the literature. The importance of the various types of PAC-Bayes bounds is justified by the connection between PAC-Bayes bounds and regularization in a learning problem. PAC-Bayesian theory provides a guarantee that upper-bounds the risk of Gibbs predictors simultaneously for all posterior distributions $Q$. PAC-Bayes bounds enjoy this property due to change of measure  inequalities. It is a well-known fact that KL-regularized objective functions are obtained from PAC-Bayes risk bounds, in which the minimizer $\hat{Q}$ is guaranteed to exist since the risk bounds hold simultaneously for all posterior distributions $Q$. The complexity term, partially controlled by the divergence, serves as a regularizer \cite{NIPS2009_3821, NIPS2006_3120, Bousquet2002}. Additionally, the link between Bayesian inference techniques and PAC-Bayesian risk bounds was shown by \cite{NIPS2016_6569}, that is, the minimization of PAC-Bayesian risk bounds maximizes the Bayesian marginal likelihood. Our results pertain to important machine learning prediction problems, such as regression, classification and structured prediction and indicate which regularizer to use. All PAC-Bayes bounds presented here are either novel, or have a tighter complexity term than existing results in the literature. 
Since our bounds are based on either $\chi^2$-divergence or $\alpha$-divergence, we excluded the comparison with the bounds based on the KL-divergence (\cite{Seeger:2003:PGE:944919.944929, Catoni2007, NIPS2019_8539, NIPS2019_9387, NIPS2016_6569, pmlr-v98-grunwald19a,  NIPS2017_7100}). Our results for sub-exponential losses are entirely novel. For the other cases, our bounds are tighter than exisiting bounds in terms of the complexity term. For instance, our bound for bounded losses is tighter than those of \cite{honorio, LucBegin} since our bound has the complexity term $\chi^2(Q\|P)^{1/4}$ and $\log(1/\delta)$, while \cite{honorio, LucBegin} have $\chi^2(Q\|P)^{1/2}$ and $1/\delta$. For sub-Gaussian loss functions, our bound has the complexity term $D_{\alpha}(Q\|P)^{1/2\alpha}$ and $\log(1/\delta)$, whereas \cite{Alquier2018} has $D_{\alpha}(Q\|P)^{1/\alpha}$ and $1/\delta$ respectively. In addition, our additive bounds, such as Equation \eqref{eq_additive_boundedvar}, have better rates than the existing bounds in \cite{Alquier2018}, since in our bound, $D_{\alpha}(Q\|P)$ and $1/\delta$ are  added, while in \cite{Alquier2018}, $D_{\alpha}(Q\|P)$ and $1/\delta$ are multiplied.

\section{Non-Asymptotic Interval for Monte Carlo Estimates}
\begin{table*}[t]
	\centering
	\caption{Summary of non-asymptotic interval for Monte Carlo estimates}
	\label{tab:4}
	\begin{tabular}{p{1.35cm}p{10cm}p{1.85cm}}	
		\textbf{Divergence} &   \textbf{Non-asymptotic interval for Monte Carlo estimates \newline} $\big|\E_{Q}[\phi (X)] - \frac{1}{n} \sum^{n}_{i=1}\phi(X_i) \big| \leq  \frac{4L^2\log(\frac{2}{\delta})}{n\gamma}+\mathcal{K}$ & \textbf{Reference}\\
		\hhline{===}
		Pseudo $\alpha$ & $\mathcal{K} = \frac{2^{\frac{2\alpha-1}{\alpha}}L\widetilde{\mathcal{D_{\alpha}}}(Q\|P)^{\frac{1}{\alpha}}}{\sqrt{\gamma}}\Gamma(\frac{3\alpha-2}{2(\alpha - 1)})^{\frac{\alpha-1}{\alpha}}$ & Proposition \ref{nonaymp_1}  \\
		
		$\chi^2$ & $\mathcal{K} =\begin{cases}\sqrt{\chi^2(Q\|P)\{\frac{1}{n}\sum_{i=1}^{n}\phi^2(x_i)+\frac{16L^2}{\gamma}\sqrt{\frac{1}{n}\log \frac{2}{\delta}}\}}$ if $ \log(\frac{2}{\delta}) \leq n\\ 
		\sqrt{\chi^2(Q\|P)\{\frac{1}{n}\sum_{i=1}^{n}\phi^2(x_i)+\frac{16L^2}{n\gamma}\log \frac{2}{\delta}\}}$ if $n < \log(\frac{2}{\delta}) 
		\end{cases}$ & Proposition \ref{nonaymp_2}  \\
		
		$KL$ &  $\mathcal{K} = KL(Q\|P)+\frac{L^2}{n\gamma}$ & Proposition \ref{nonaymp_kl}\\
		
	\end{tabular}
\end{table*}

We now turn to another application of change of measure inequalities. We will introduce a methodology that enables us to find a non-asymptotic interval for Monte Calro (MC) estimate. All results shown in this section are entirely novel and  summarized in Table \ref{tab:4}. Under some circumstances, our methodology could be a promising alternative to existing MC methods (e.g., Importance Sampling and Rejection Sampling \cite{10.5555/1051451}) because their non-asymptotic intervals are hard to analyze. In this section, we consider a problem to estimate an expectation of an Lipschitz function $\phi: \R^d \rightarrow \R$ with respect to a complicated distribution $Q$, namely $\E_{Q}[\phi (X)]$ by the sample mean $\frac{1}{n}\sum^n_{i=1}\phi(X_i)$ over a distribution $P$, where $Q$ is any distribution we are not able to sample from. For a motivating application,  consider $Q$ being a probabilistic graphical model (see e.g., \cite{honorio_pgm}).  Assume that we have a strongly log-concave distribution $P$ with a parameter $\gamma$ for a sampling distribution (see Appendix \ref{appendix_logconcave} for definition). Under the above conditions, we claim the following proposition.
\begin{proposition}[A general expression of non-asymptotic interval for Monte Carlo estimates]
	\label{nonaymp_1}
	Let $X$ be a d-dimensional random vector.  Let $P$ and $Q$ be the probability measures induced by $X$. Let $P$ be a strongly log-concave distribution with parameter $\gamma>0$. Let $\phi : \R^d \rightarrow \R$ be any L-Lipschitz function with respect to the Euclidean norm. Suppose we draw i.i.d. samples $X_1, X_2, ..., X_n \sim P$. For $\alpha>1$, with probability at least $1-\delta$,  we have
	\begin{equation*}
	\bigg|\E_{Q}[\phi (X)] - \frac{1}{n} \sum^{n}_{i=1}\phi(X_i) \bigg| \leq  \frac{4L^2\log(\frac{2}{\delta})}{n\gamma}+\frac{2^{\frac{2\alpha-1}{\alpha}}L\widetilde{\mathcal{D_{\alpha}}}(Q\|P)^{\frac{1}{\alpha}}}{\sqrt{\gamma}}\Gamma\bigg(\frac{3\alpha-2}{2(\alpha - 1)}\bigg)^{\frac{\alpha-1}{\alpha}}
	\end{equation*}	
\end{proposition}
The second term on the right hand side indicates a bias of an empirical mean $\frac{1}{n}\sum^n_{i=1}\phi(X_i)$ under the sampling distribution $P$.  Next, we present a more informative bound.
\begin{proposition}[$\chi^2$-based expression of non-asymptotic interval for Monte Carlo estimates]
	\label{nonaymp_2}
	Let $X, P, Q, L, \gamma$ and $\phi$ be defined as in Proposition \ref{nonaymp_1} . Suppose we have i.i.d. samples $X_1, X_2, ..., X_n \sim P$. Then, with probability at least $(1-\delta)^2$,  we have
	\begin{equation*}
	\bigg|\E_{Q}[\phi (X)] - \frac{1}{n} \sum^{n}_{i=1}\phi(X_i) \bigg| \leq \frac{4L^2\log(\frac{2}{\delta})}{n\gamma} + \mathcal{K}
	\end{equation*}	
	where
	\begin{equation*}
	\begin{split}
	\mathcal{K} = \begin{cases}\sqrt{\chi^2(Q\|P)\{\frac{1}{n}\sum_{i=1}^{n}\phi^2(X_i)+\frac{16L^2}{\gamma}\sqrt{\frac{1}{n}\log \frac{2}{\delta}}\}} \quad \text{if }  \log(\frac{2}{\delta}) \leq n\\ 
	\sqrt{\chi^2(Q\|P)\{\frac{1}{n}\sum_{i=1}^{n}\phi^2(X_i)+\frac{16L^2}{n\gamma}\log \frac{2}{\delta}\}}  \quad \text{if } n < \log(\frac{2}{\delta}) 
	\end{cases}
	\end{split}
	\end{equation*}	
\end{proposition}
This result provides an insight into how good a proposal distribution $P$ is, meaning that the effect of the deviation between $P$ and $Q$, namely $\chi^2(Q\|P)$, is inflated by the empirical variance. This implication might support some results in the literature (\cite{Cornebise, NIPS2017_6866}). So far we have considered the pseudo $\alpha$-divergence as well as $\chi^2$ divergence.  \cite{Duembgen} presented an approximation for an arbitrary distribution $Q$ by distributions with log-concave density with respect to the KL divergence, which motivates the following result.
\begin{proposition}[KL-based expression of non-asymptotic interval for Monte Carlo estimates]
	\label{nonaymp_kl}
	Let $X, P, Q, L, \gamma$ and $\phi$ be defined as in Proposition \ref{nonaymp_1} . Suppose we have i.i.d. samples $X_1, X_2, ..., X_n \sim P$. Then, with probability at least $1-\delta$,  we have
	\begin{equation*}
	\bigg|\E_{Q}[\phi (X)] - \frac{1}{n} \sum^{n}_{i=1}\phi(X_i) \bigg| \leq
	\frac{4L^2\log(\frac{2}{\delta})}{n\gamma} + KL(Q\|P)+\frac{L^2}{n\gamma}
	\end{equation*}	
\end{proposition}
This result shows that, under the assumption of Proposition \ref{nonaymp_kl}, the empirical mean $\frac{1}{n}\sum^n_{i=1}\phi(X_i)$ over the sampling distribution $P$ aymptotically differs from $\E_{Q}[\phi (X)]$ by $KL(Q\|P)$ at most.

\bibliography{changeofmeasure}
\clearpage
\appendix

\section{Detailed Proofs}
\label{appendix_sec2}
\subsection{Proof of Corollary  \ref{cor_cmi_unconstrainedfdiv}}
\begin{proof}
	Removing the supremum and rearranging the terms in Lemma \ref{lemma_var_rep_f}, we prove our claim.
	\qedhere
\end{proof}

\subsection{Proof of Lemma  \ref{lemma_cmi_constrainedchi2}}
\begin{proof}
	By Theorem \ref{theorem_cmi_fdiv}, we want to find 
	\begin{equation}
	\label{chi_constrainedterm}
	\begin{split}
	(\mathbb{I}^{R}_{f, P})^*(\phi)& =\sup_{p\in \Delta(P)}{\int_{\mathcal{H}} \phi p - (p-1)^2dP}\\
	& = \sup_{p\in \Delta(P)}{\E_P [\phi p - (p-1)^2]}
	\end{split}
	\end{equation}
	where $\phi$ is a measurable function on $\mathcal{H}$.
	In order to find the supremum on the right hand side, we consider the following Lagrangian:
	\begin{equation*}
	\begin{split}
	L(p, \lambda)&= \E_P[\phi p - (p-1)^2] + \lambda (\E_P[p]-1)\\
	&=\int_{\mathcal{H}} \phi p - (p-1)^2 dP + \lambda \bigg(\int_{\mathcal{H}} p dP - 1\bigg)
	\end{split}
	\end{equation*}
	where $\lambda \in \mathbb{R}$ and $p$ is constrained to be a probability density over $\mathcal{H}$. By talking the functional derivative with respect to $p$ and renormalize by dropping the $dP$ factor multiplying all terms and setting it to zero, we have $\frac{\partial L}{\partial p dP} = \phi -2(p-1) + \lambda = 0$.
	Thus, we have $p = \frac{\phi + \lambda}{2} + 1$. Since $p$ is constrained to be $\int_{\mathcal{H}} p dP = 1$, we have $\int_{\mathcal{H}} (\frac{\phi + \lambda}{2} + 1) dP = 1$. Then, we obtain $\lambda= -\E_P[\phi]$ and the optimum $p = \frac{\phi - \E_P[\phi] }{2} + 1$. Plugging it in Equation \eqref{chi_constrainedterm}, we have 
	\begin{equation*}
	\begin{split}
	(\mathbb{I}^{R}_{f, P})^*(\phi) & = {E_P \bigg[\phi \bigg(\frac{\phi - \E_P[\phi] }{2} + 1\bigg) - \bigg(\frac{\phi - \E_P[\phi] }{2}\bigg)^2\bigg]}
	\end{split}
	\end{equation*}
	which simplifies to $\E_P[\phi]+\frac{1}{4}\Var_P[\phi]$. 
	\qedhere
\end{proof}

\subsection{Proof of Lemma  \ref{lemma_cmi_unconstrainchi2}}
\begin{proof}
	Notice that the $\chi^2$-divergence is obtained by setting $f(x) = (x-1)^2$. In order to find the convex conjugate $f^*(y)$ from Equation \eqref{eq_conjugate}, let $g(x,y)=xy-(x-1)^2$. We need to find the supremum of $g(x,y)$ with respect to $x$.  By using differentiation with respect to $x$ and setting the derivative to zero, we have $\frac{\partial g}{\partial x}=y-2(x-1)=0$. Thus, plugging $x=\frac{y}{2}+1$ for $g(x,y)=xy-(x-1)^2$, we obtain $f^*(y) =y+\frac{y^2}{4}$.
	Plugging $f(x)$ and $f^*(y)$ in Corollary \ref{cor_cmi_unconstrainedfdiv}, we prove our claim.
	\qedhere
\end{proof}

\subsection{Proof of Lemma  \ref{lemma_cmi_tv}}
\begin{proof}
	By Theorem \ref{theorem_cmi_fdiv}, we want to find 
	\begin{equation*}
	\begin{split}
	(\mathbb{I}^{R}_{f, P})^*(\phi) & = \sup_{p\in \Delta(P)}{\E_P [\phi p - \frac{1}{2}|p-1|]}
	\end{split}
	\end{equation*}
	In order to find the supremum on the right hand side, we consider the following Lagrangian
	\begin{equation*}
	\begin{split}
	L(p, \lambda)& = \E_P[\phi p - \frac{1}{2}|p-1|)] + \lambda (\E_P[p] - 1)\\
	& = \begin{cases} 
	\E_P[\phi p - \frac{1}{2}p+\frac{1}{2}] + \lambda (\E_P[p] - 1)& \text{if } 1\leq p\\
	\E_P[\phi p + \frac{1}{2}p-\frac{1}{2}] + \lambda (\E_P[p] - 1)&\text{if } 0 < p < 1\\ 
	\end{cases}
	\end{split}
	\end{equation*}
	Then, it is not hard to see that if $|\phi|\leq \frac{1}{2}$,  $L(p, \lambda)$ is maximized at $p=1$, otherwise $L\rightarrow \infty$ as $p\rightarrow1$. Thus, Lemma \ref{lemma_cmi_tv} holds for $|\phi|\leq \frac{1}{2}$. If we add $\frac{1}{2}$ on the both sides, then $\phi$ is bounded between 0 and 1, as we claimed in the corollary.
	\qedhere
\end{proof}

\subsection{Proof of Lemma  \ref{lemma_cmi_unconstrainedalpha}}
\begin{proof} The corresponding convex function $f(x)$ for the $\alpha$-divergence is defined as $f(x)=\frac{x^{\alpha}-1}{\alpha(\alpha-1)}$. Applying the same procedure as in Lemma \ref{lemma_cmi_unconstrainchi2}, we get the convex conjugate $f^*(y)=\frac{(\alpha-1)^{\frac{\alpha}{\alpha-1}}}{\alpha}y^{\frac{\alpha}{\alpha-1}} + \frac{1}{\alpha(\alpha-1)}$ for $\alpha>1$. 
	Plugging $f(x)$ and $f^*(y)$ in Corollary \ref{cor_cmi_unconstrainedfdiv}, we prove our claim.
	\qedhere
\end{proof}

\subsection{Proof of Lemma  \ref{lemma_cmi_unconstrainedhel}}
\begin{proof} The corresponding convex function $f(x)$ for the squared Hellinger divergence is defined as $f(x)=(\sqrt{x}-1)^2$. Let $g(x,y)=xy-(\sqrt{x}-1)^2$. We have $\frac{\partial g}{\partial x}=y-1+\frac{1}{\sqrt{x}}=0$. Since we consider $x>0$, $y<1$. Applying the same procedure as in Lemma \ref{lemma_cmi_unconstrainchi2}, we get the convex conjugate $f^*(y)=\frac{y}{1-y}$. 
	Plugging $f(x)$ and $f^*(y)$ in Corollary \ref{cor_cmi_unconstrainedfdiv}, we prove our claim.
	\qedhere
\end{proof}

\subsection{Proof of Lemma  \ref{lemma_cmi_unconstrainedrevkl}}
\begin{proof}
	
	The corresponding convex function $f(x)$ for the reverse KL-divergence is defined as $f(x)=-\log(x)$. Applying the same procedure as in Lemma \ref{lemma_cmi_unconstrainedhel}, we get the convex conjugate $f^*(y)=\log(-\frac{1}{y})-1$. 
	Plugging $f(x)$ and $f^*(y)$ in Corollary \ref{cor_cmi_unconstrainedfdiv}, we have
	\begin{equation*}
	\E_Q[\psi]\leq \overline{KL}(Q\|P) + \E_P\bigg[\log\bigg(-\frac{1}{\psi}\bigg)\bigg] - 1
	\end{equation*}
	
	where $\psi < 0$. Letting $\psi = \phi - 1$, we prove our claim.
	\qedhere
\end{proof}

\subsection{Proof of Lemma  \ref{lemma_cmi_unconstrainedrevchi2}}
\begin{proof} The corresponding convex function $f(x)$ for the Neyman $\chi^2$-divergence is defined as $f(x)=-\frac{(x-1)^2}{x}$. Applying the same procedure as in Lemma \ref{lemma_cmi_unconstrainedhel}, we get the convex conjugate $f^*(y)=2-2\sqrt{1-y}$. 
	Plugging $f(x)$ and $f^*(y)$ in Corollary \ref{cor_cmi_unconstrainedfdiv}, we prove our claim.
	\qedhere
\end{proof}

\subsection{Proof of Lemma  \ref{lemma_mulalpha}}
\begin{proof}
	\begin{equation*}
	\begin{split}
	\E_Q[\phi] &= \int_{\mathcal{H}} \phi dQ \\
	&= \int_{\mathcal{H}} \phi \frac{dQ}{dP}dP \\
	& \leq \bigg(\int_{\mathcal{H}}\bigg|\frac{dQ}{dP}\bigg|^{\alpha}dP\bigg)^{\frac{1}{\alpha}}\bigg(\int_{\mathcal{H}}|\phi|^{\frac{\alpha}{\alpha-1}}dP\bigg)^{\frac{\alpha-1}{\alpha}}\\
	&= \bigg(\alpha(\alpha-1)D_{\alpha}(Q\|P)+1\bigg)^{\frac{1}{\alpha}}\bigg(\E_P[|\phi|^{\frac{\alpha}{\alpha-1}}]\bigg)^{\frac{\alpha-1}{\alpha}}
	\end{split}
	\end{equation*}
	The third line is due to the H\"{o}lder's inequality.
	\qedhere
\end{proof}

\subsection{Proof of Lemma  \ref{lemma_generalhcr}}
\begin{proof}
	Consider the covariance of $\phi$ and $\frac{dQ}{dP}$. 
	\begin{equation*}
	\begin{split}
	\bigg|Cov_P\bigg(\phi, \frac{dQ}{dP}\bigg) \bigg|&= \bigg| \int_{\mathcal{H}}\phi\frac{dQ}{dP}dP - \int_{\mathcal{H}}\phi dP\int_{\mathcal{H}}\frac{dQ}{dP}dP\bigg|\\
	&=\bigg|\E_Q[\phi]-\E_P[\phi]\bigg|
	\end{split}
	\end{equation*}
	On the other hand,
	\begin{equation*}
	\begin{split}
	&\bigg|Cov_P\bigg(\phi, \frac{dQ}{dP}\bigg)\bigg|\\
	&=\bigg|\E_P\bigg[(\phi -\mu_P)\bigg(\frac{dQ}{dP}-\E_P\bigg[\frac{dQ}{dP}\bigg]\bigg)\bigg]\bigg|\\
	&\leq\E_P\bigg[\bigg|(\phi -\mu_P)\bigg(\frac{dQ}{dP}-\E_P\bigg[\frac{dQ}{dP}\bigg]\bigg)\bigg|\bigg]\\
	&=\int_{\mathcal{H}}\bigg|(\phi -\mu_P)\bigg(\frac{dQ}{dP}-1\bigg)\bigg|dP\\
	&\leq\bigg(\int_{\mathcal{H}}\bigg|\frac{dQ}{dP}-1\bigg|^{\alpha}dP\bigg)^{\frac{1}{\alpha}}\bigg(\int_{\mathcal{H}}|\phi-\mu_P|^{\frac{\alpha}{\alpha-1}}dP\bigg)^{\frac{\alpha-1}{\alpha}}\\
	&\leq \widetilde{\mathcal{D_{\alpha}}}(Q\|P)^{\frac{1}{\alpha}}\bigg(\E_P[|\phi-\mu_P|^{\frac{\alpha}{\alpha-1}}]\bigg)^{\frac{\alpha-1}{\alpha}}
	\end{split}
	\end{equation*}
	which proves our claim.
	\qedhere
\end{proof}

\subsection{Proof of Proposition  \ref{prop_pac_boundedvar}}
\begin{proof}
	Suppose the  convex function $\Delta: \R\times\R \rightarrow \R$ is defined as in Proposition \ref{prop_boundloss}. By Chebyshev's inequality, for any $\epsilon > 0$ and any $h \in \mathcal{H}$,
	\begin{equation*}
	\begin{split}
	&\mathop{Pr}_{(x,y) \sim D}(\phi_D(h)\geq\epsilon) \\ 
	&=\mathop{Pr}_{(x,y) \sim D}\bigg((R_S(h)-R_D(h))^2\geq\frac{\epsilon}{t}\bigg) \\
	&=\mathop{Pr}_{(x,y) \sim D}\bigg(\bigg|\frac{1}{m} \sum_{i=1}^{m} \ell(h(x_i), y_i)-\mathop{\E}_{(x,y) \sim D} \ell(h(x), y)\bigg|\geq\sqrt{\frac{\epsilon}{t}}\bigg) \\
	& \leq \frac{t\sigma^2}{m\epsilon}
	\end{split}
	\end{equation*}
	Setting $\delta=\frac{t\sigma^2}{m\epsilon}$, we have
	\begin{equation*}
	\begin{split}
	&\phi_D(h) \leqdel \frac{t\sigma^2}{m\delta}
	\end{split}
	\end{equation*}
	Now, we have the upper bound for $\phi_D(h)$ so we can apply the same procedure as in Proposition \ref{prop_boundloss}.
	\qedhere
\end{proof} 

\subsection{Proof of Proposition  \ref{nonaymp_1}}
\begin{proof}
	Note that we have Lemma \ref{lemma_generalhcr}. Now, it remains to bound $\E_P[\phi(X)]$ and $\E_P[|\phi(X)-\E_P[\phi(X)]|^{\frac{\alpha}{\alpha-1}}]$. Since $P$ is a strongly log-concave distribution and $\phi$ is an L-Lipschitz function, Theorem 3.16 in \cite{wainwright_2019} implies $\phi(X)-\E_P[\phi(X)]$ is a sub-Gaussian random variable with parameter $\sigma^2=\frac{2L^2}{\gamma}$. Therefore, we have
	\begin{equation*}
	\begin{split}
	Pr\bigg(\bigg|\frac{1}{n}\sum^n_{i=1}\phi(X_i)-\E_P[\phi (X)]\bigg| \geq \epsilon \bigg) \leq 2e^{\frac{n\gamma \epsilon^2}{4L^2}}
	\end{split}
	\end{equation*}
	Thus, 
	\begin{equation*}
	\bigg| \E_P[\phi (X)] - \frac{1}{n}\sum^n_{i=1}\phi(X_i)\bigg|  \leqdel  \frac{4L^2}{n \gamma}\log(\frac{2}{\delta})
	\end{equation*}
	On the other hand, letting $U=\phi(X) -\E_P[\phi(X)]$ , then we have
	\begin{equation*}
	\begin{split}
	\E_{P}[|U|^{\frac{\alpha}{\alpha-1}}] & = \int_0^{\infty}Pr\{|U|^{\frac{\alpha}{\alpha-1}}>u\}du \\
	& = \int_0^{\infty}Pr\{|U|>u^{\frac{\alpha-1}{\alpha}}\}du \\
	&= \frac{\alpha}{\alpha-1}\int_0^{\infty}u^{\frac{1}{\alpha-1}}Pr\{|U|>u\}du\\
	&= \frac{\alpha}{\alpha-1}\int_0^{\infty}u^{\frac{1}{\alpha-1}}Pr\{|\phi(X) -\E_P[\phi(X)]|>u\}du\\
	&\leq \frac{\alpha}{\alpha-1}\int_0^{\infty}u^{\frac{1}{\alpha-1}}2e^{-\frac{\gamma u^2}{4L^2}}du\\
	&= \frac{\alpha}{\alpha-1}\int_0^{\infty}u^{\frac{\alpha}{2(\alpha-1)}-1}e^{-\frac{\gamma u}{4 L^2}}du\\
	&= \frac{\alpha}{\alpha-1}\Gamma\bigg(\frac{\alpha}{2(\alpha-1)}\bigg) \bigg(\frac{4 L^2}{\gamma}\bigg)^{\frac{\alpha}{2(\alpha-1)}}\\
	&= 2\Gamma\bigg(\frac{3\alpha-2}{2(\alpha-1)}\bigg) \bigg(\frac{4 L^2}{\gamma}\bigg)^{\frac{\alpha}{2(\alpha-1)}}\\
	\end{split}
	\end{equation*}
	The inequality is due to sub-Gaussianity. Putting everything together with Lemma \ref{lemma_generalhcr}, we prove our claim.
	\qedhere
\end{proof}

\subsection{Proof of Proposition  \ref{nonaymp_2}}
\begin{proof} By Lemma \ref{lemma_hcr}, we have
	\begin{equation*}
	\begin{split}
	|E_Q[\phi(X)] - \E_P[\phi(X)]| \leq \sqrt{\chi^2(Q\|P) \E_P[\phi^2(X)]}
	\end{split}
	\end{equation*}
	where $\E_P[\phi(X)]$ is bounded as in Proposition \ref{nonaymp_1}.  
	\begin{equation*}
	\bigg| \E_P[\phi (X)] - \frac{1}{n}\sum^n_{i=1}\phi(X_i)\bigg|  \leqdel  \frac{4L^2}{n \gamma}\log(\frac{2}{\delta})
	\end{equation*}
	It remains to bound $E_P[\phi^2(X)]$. Note that $U^2$ is a sub-exponential random variable with parameters $(4\sqrt{2}\sigma^2, 4\sigma^2)$ when $U$ is a sub-Gaussian random variable with a parameter $\sigma$ (See Appendix B in \cite{honorio} for details). Therefore,
	\begin{equation*}
	\begin{split}
	Pr\bigg(\bigg|\frac{1}{n}\sum^n_{i=1}\phi^2(X_i)-\E_P[\phi^2 (X)]\bigg|  \geq \epsilon \bigg) \leq
	\begin{cases} 
	2e^{-\frac{n\gamma^2 \epsilon^2}{256L^4}} \quad \text{if } \log \frac{2}{\delta} \leq n\\
	2e^{-\frac{n\gamma^2 \epsilon^2}{256L^4}} \quad \text{if } \log \frac{2}{\delta} > n
	\end{cases} 
	\end{split}
	\end{equation*}
	Thus, 
	\begin{equation*}
	\begin{split}
	\bigg| \E_P[\phi^2 (X)] - \frac{1}{n}\sum^n_{i=1}\phi^2(X_i)\bigg|  \leqdel  
	\begin{cases}
	\frac {16L^2}{n}\sqrt{\frac{1}{n}\log(\frac{2}{\delta})}  \quad \text{if } \log \frac{2}{\delta} \leq n\\
	\frac {16L^2}{n \gamma}\log(\frac{2}{\delta}) \quad \text{if } \log \frac{2}{\delta} > n
	\end{cases}
	\end{split}
	\end{equation*}
	Putting everything together, we prove our claim.
	\qedhere
\end{proof}

\subsection{Proof of Proposition  \ref{nonaymp_kl}}
\begin{proof} 	Note that we have the following change of measure inequality for any function $\psi:\R \rightarrow \R$ from the Donsker-Varadhan representation for the KL-divergence.
	\begin{equation}
	\label{eq_donsker_varadhan}
	\E_Q[\psi(X)]\leq  KL(Q\|P)+\log(\E_P[e^{\psi(X)}])
	\end{equation}
	Suppose $\psi_1:\R \rightarrow \R$ satisfies the following condition.
	\begin{equation*}
	\psi_1(X) = \phi(X) - \E_P[\phi(X)]
	\end{equation*}
	By plugging in Equation \eqref{eq_donsker_varadhan}, we have 
	\begin{equation*}
	\begin{split}
	\E_Q[\phi(X)]-\E_P[\phi(X)] & \leq KL(Q\|P)+\log(\E_P[e^{\phi(X)-\E_P[\phi(X)]}])
	\end{split}
	\end{equation*}
	On the other hand, suppose $\psi_2:\R \rightarrow \R$ satisfies the following condition.
	\begin{equation*}
	\psi_2(X) =  \E_P[\phi(X)] - \phi(X)
	\end{equation*}
	By plugging in Equation \eqref{eq_donsker_varadhan}, we have 
	\begin{equation*}
	\begin{split}
	\E_P[\phi(X)]-\E_Q[\phi(X)] & \leq KL(Q\|P)+\log(\E_P[e^{\E_P[\phi(X)]-\phi(X)}])
	\end{split}
	\end{equation*}
	By putting all together, we have
	\begin{equation*}
	\begin{split}
	&|\E_Q[\phi(X)] - \E_P[\phi(X)]| \leq  KL(Q\|P) + \log(\E_P[e^{\phi(X)-\E_P[\phi(X)]}]) \vee \log(\E_P[e^{\E_P[\phi(X)]-\phi(X)}])
	\end{split}
	\end{equation*}
	Suppose $X_1, X_2, \dots, X_n$ are i.i.d. random variables.We can easily see that we have
	\begin{equation}
	\label{eq_kl_mc1}
	\begin{split}
	|\E_Q[\phi(X)] - \E_P[\phi(X)]| & \leq  KL(Q\|P) \\
	&+ \log(\E_P[e^{\frac{1}{n}\sum_{i=1}^{n}\phi(X_i)-\E_P[\phi(X)]}]) \vee \log(\E_P[e^{\E_P[\phi(X)]-\frac{1}{n}\sum_{i=1}^{n}\phi(X_i)}])
	\end{split}
	\end{equation}
	$\E_P[\phi(X)]$ is bounded as in Proposition \ref{nonaymp_1}.  
	\begin{equation}
	\label{eq_kl_mc2}
	\bigg| \E_P[\phi (X)] - \frac{1}{n}\sum^n_{i=1}\phi(X_i)\bigg|  \leqdel  \frac{4L^2}{n \gamma}\log(\frac{2}{\delta})
	\end{equation}
	As shown in Proposition \ref{nonaymp_1},  $\phi(X)-\E_P[\phi(X)]$ is a sub-Gaussian random variable with parameter $\sigma^2=\frac{2L^2}{\gamma}$. Therefore, by Definition \ref{def_subgaussian},
	\begin{equation}
	\label{eq_kl_mc3}
	\log(\E_P[e^{\frac{1}{n}\sum_{i=1}^{n}\phi(X_i)-\E_P[\phi(X)]}]) \vee \log(\E_P[e^{\E_P[\phi(X)]-\frac{1}{n}\sum_{i=1}^{n}\phi(X_i)}])  \leq  \frac{L^2}{n\gamma}
	\end{equation}
	since both values in the maximum are bounded by $\frac{L^2}{n\gamma}$.
	By \eqref{eq_kl_mc1}, \eqref{eq_kl_mc2} and \eqref{eq_kl_mc3}, we prove our claim.
	\qedhere
\end{proof}

\section{Pseudo $\alpha$-divergence is a member of the family of $f$-divergences.}
\label{appendix_alpha_is_f}
\begin{proof}
	Let $f(t)=|t-1|^{\alpha}$ and $\lambda\in[0,1]$. Let $\frac{1}{\alpha}+\frac{1}{\beta}=1$ for $\alpha \geq 1$.
	\begin{equation*}
	\begin{split}
	&f(\lambda x + (1-\lambda) y) \\
	&= |\lambda x + (1-\lambda) y - 1|^{\alpha}\\
	&=|\lambda (x-1) + (1-\lambda)(y - 1)|^{\alpha}\\
	&\leq \{\lambda|x-1| + (1-\lambda)|y - 1|\}^{\alpha}\\
	&= \{\lambda^{\frac{1}{\alpha}}|x-1|\lambda^{\frac{1}{\beta}} + (1-\lambda)^{\frac{1}{\alpha}}|y - 1|(1-\lambda)^{\frac{1}{\beta}}\}^{\alpha}\\
	&\leq (\lambda|x-1|^{\alpha}+(1-\lambda)|y-1|^{\alpha})(\lambda +(1-\lambda))^{\frac{\alpha}{\beta}}\\
	&= \lambda|x-1|^{\alpha}+(1-\lambda)|y-1|^{\alpha}\\
	&=\lambda f(x) +(1-\lambda) f(y)
	\end{split}
	\end{equation*}
	The first inequality is due to the triangle inequality and the second inequality is due to H\"{o}lder's inequality. By the definition of convexity, $f(t)$ is convex. Also, $f(1)=0$. By the definition of $f$-divergence, we prove our claim.
	\qedhere
\end{proof}

\section{PAC-Bayesian bounds for bounded, sub-Gaussian and sub-exponential losses}
\label{appendix_pac_bayes}
Here we complement our results in Section \ref{pacbayesianbounds},  by showing our PAC-Bayesian  generalization bounds for bounded, sub-Gaussian and sub-exponential losses.

\subsection{Bounded Loss Function}
\label{sub_boundedloss}
First, let us assume here that the loss function is bounded, i.e., for any $h \in \mathcal{H}$ and \((x,y) \in \mathcal{X} \times \mathcal{Y}\), $\ell(h(x),y) \in [0,R]$ for $R >0 $. Note that, for $R>1$, we cannot use the total variation, squared Hellinger, Reverse KL and Neyman $\chi^2$ divergence because $\phi$ is constrained to be in $[0,1]$.
\begin{proposition}[The PAC-Bayesian bounds for bounded loss function]
	\label{prop_boundloss}
	Let $P$ be any prior distribution over an infinite hypothesis space $\mathcal{H}$. For a given posterior distribution $Q$ over an infinite hypothesis space $\mathcal{H}$, let $R_D(G_Q)$ and $R_S(G_Q)$ be the Gibbs risk and the empirical Gibbs risk as in Equation \eqref{gibbsrisk} and \eqref{gibbsempiricalrisk} respectively. 
	For the sample size $m>0$ and $\alpha > 1$, with probability at least $1-\delta$, simultaneously for all posterior distributions $Q$, we have
	\begin{equation}
	\label{prop_bound1}
	\begin{split}
	R_D(G_Q)  \leq R_S(G_Q)+\sqrt{\frac{R^2}{2m}\log(\frac{2}{\delta})\big(\alpha(\alpha-1)D_{\alpha}(Q\|P)+1\big)^{\frac{1}{\alpha}}}
	\end{split}
	\end{equation}
	\begin{equation}
	\label{prop_bound2}
	\begin{split}
	R_D(G_Q)  \leq R_S(G_Q)+\sqrt{\frac{1}{m}\bigg(D_{\alpha}(Q\|P)+\frac{1}{\alpha(\alpha-1)}\bigg)+\frac{1}{m\alpha}\bigg(\frac{R^2(\alpha-1)}{2}\log(\frac{2}{\delta})\bigg)^{\frac{\alpha}{\alpha-1}}}
	\end{split}
	\end{equation}	
\end{proposition}
\begin{proof}
	Suppose that we have a convex function $\Delta: \R\times\R \rightarrow \R$, that measures the discrepancy between the observed empirical Gibbs risk $R_S(G_Q)$ and the true Gibbs risk $R_D(G_Q)$ on distribution $Q$. Given that, the purpose of the PAC-Bayesian theorem is to upper-bound the discrepancy $t\Delta(R_D(G_Q), R_S(G_Q))$ for any $t>0$. 
	Let $\phi_D(h):=t\Delta(R_D(h), R_S(h))$, where the subscript of $\phi_D$ shows the dependency on the data distribution $D$. Let $\Delta(q,p)=(q-p)^2$.
	By applying Jensen's inequality on convex function $\Delta$ for the first step,
	\begin{equation}
	\label{jensenrisk}
	\begin{split}
	t\Delta(R_D(G_Q), R_S(G_Q)) &= t\Delta(\mathop{\mathbb{E}}_{h \sim Q}R_D(h), \mathop{\mathbb{E}}_{h \sim Q}R_S(h))\\
	&\leq \mathop{\mathbb{E}}_{h \sim Q}t\Delta(R_D(h), R_S(h))\\
	&= \mathop{\mathbb{E}}_{h \sim Q}\phi_D(h)
	\end{split}
	\end{equation}
	where  $\phi_D(h) = t\Delta(R_D(h), R_S(h)) = t(R_S(h)-R_D(h))^2$.
	By Hoeffding's inequality, for any $\epsilon > 0$ and any $h \in \mathcal{H}$,
	\begin{equation*}
	\begin{split}
	&\mathop{Pr}_{(x,y) \sim D}(\phi_D(h)\geq\epsilon) \\ 
	&=\mathop{Pr}_{(x,y) \sim D}\bigg(\big(R_S(h)-R_D(h)\big)^2\geq\frac{\epsilon}{t}\bigg) \\
	&=\mathop{Pr}_{(x,y) \sim D}\bigg(\bigg|\frac{1}{m} \sum_{i=1}^{m} \ell(h(x_i), y_i)-\mathop{\E}_{(x,y) \sim D} \ell(h(x), y)\bigg|\geq\sqrt{\frac{\epsilon}{t}}\bigg) \\
	& \leq 2e^{-\frac{2m\epsilon}{R^2t}}
	\end{split}
	\end{equation*}
	Setting $\delta = 2e^{-\frac{2m\epsilon}{R^2t}}$, we have
	\begin{equation*}
	\begin{split}
	\phi_D(h) \leqdel \frac{tR^2}{2m}\log(\frac{2}{\delta})
	\end{split}
	\end{equation*}
	The symbol \(\displaystyle{\leqdel}\) denotes that the inequality holds with probability at least $1-\delta$. The second line holds due to the Hoeffding's inequality. For $\alpha>1$, we have
	\begin{equation*}
	\begin{split}
	(\phi_D(h))^{\frac{\alpha}{\alpha-1}} \leqdel \bigg(\frac{tR^2}{2m}\log(\frac{2}{\delta})\bigg)^{\frac{\alpha}{\alpha-1}}
	\end{split}
	\end{equation*}
	Also note that 
	\begin{equation}
	\label{momentbound}
	\begin{split}
	\mathop{\E}_{h\sim P}[(\phi_D(h))^{\frac{\alpha}{\alpha-1}}] &\leq \sup_{(x,y) \sim D}\bigg\{\mathop{\E}_{h\sim P}[(\phi_D(h))^{\frac{\alpha}{\alpha-1}}]\bigg\} \\
	&\leq \mathop{\E}_{h\sim P}\bigg[\sup_{(x,y) \sim D}(\phi_D(h))^{\frac{\alpha}{\alpha-1}}\bigg] \\ 
	&\leqdel \bigg(\frac{tR^2}{2m}\log(\frac{2}{\delta})\bigg)^{\frac{\alpha}{\alpha-1}}
	\end{split}
	\end{equation}
	By applying Equations \eqref{jensenrisk} and \eqref{momentbound} to Lemma \ref{lemma_mulalpha},
	\begin{equation*}
	\begin{split}
	t(R_D(G_Q)-R_S(G_Q))^2 \leqdel  \frac{tR^2}{2m}\log(\frac{2}{\delta})\bigg(\alpha(\alpha-1)D_{\alpha}(Q\|P)+1\bigg)^{\frac{1}{\alpha}}
	\end{split}
	\end{equation*}
	which proves Equation \eqref{prop_bound1}. By applying Equations \eqref{jensenrisk} and \eqref{momentbound} to Lemma \ref{lemma_cmi_unconstrainedalpha} and setting $t=m$, we prove Equation \eqref{prop_bound2}.
	\qedhere
\end{proof}

Equation \eqref{prop_bound1} has a tighter bound than Proposition 2 in \cite{Alquier2018}. Also, Equation \eqref{prop_bound2} has a novel expression of PAC-Bayesian theorem with $\chi^2$-divergence. The same arguments apply to the bounds in Proposition \ref{prop_pac_subgauss}, \ref{prop_pac_subexp} and \ref{prop_pac_boundedvar}.

The following corollary is an immediate consequence of this proposition for $\alpha=2$.
\begin{corollary}[The PAC-Bayesian bounds with $\chi^2$-divergence for bounded loss function]
	\label{cor_pac_chi_bounded_loss}
	Let $P$ be any prior distribution over an infinite hypothesis space $\mathcal{H}$. For a given posterior distribution $Q$ over an infinite hypothesis space $\mathcal{H}$, let $R_D(G_Q)$ and $R_S(G_Q)$ be the Gibbs risk and the empirical Gibbs risk as in Equation \eqref{gibbsrisk} and \eqref{gibbsempiricalrisk} respectively. 
	For the sample size $m>0$, with probability at least $1-\delta$, simultaneously for all posterior distributions $Q$, we have
	\begin{equation}
	\label{cor_bounded_loss_1}
	\begin{split}
	R_D(G_Q) &  \leq  R_S(G_Q) + \sqrt{\frac{R^2}{2m}\log(\frac{2}{\delta})\sqrt{\chi^2(Q\|P)+1}}
	\end{split}
	\end{equation}	
	\begin{equation}
	\label{cor_bounded_loss_2}
	\begin{split}
	R_D(G_Q)  \leq  R_S(G_Q) + \sqrt{\frac{1}{2m}\bigg(\chi^2(Q\|P)+1+\bigg(\frac{R^2}{2}\log(\frac{2}{\delta})\bigg)^2\bigg)}
	\end{split}
	\end{equation}
\end{corollary}

These are novel PAC-Bayesian bounds for bounded loss functions based on $\chi^2$-divergence. Most PAC-Bayesian bounds for bounded loss functions are based on the KL-divergence for the complexity term (e.g., \cite{Catoni2007, Seeger:2003:PGE:944919.944929, McAllester98somepac-bayesian}). \cite{honorio} and \cite{LucBegin} contain bounds for bounded loss functions with $\chi^2$-divergence. 
Compared to their bounds, the bound \eqref{cor_bounded_loss_1} is tighter due to the power of $\frac{1}{4}$ on the complexity term and $\log(\frac{1}{\delta})$ instead of $\frac{1}{\delta}$. Also, PAC-Bayes bound \eqref{cor_bounded_loss_2} has a unique characteristics; $\chi^2(Q\|P)$ and $\frac{1}{\delta}$ are independent since $\chi^2(Q\|P)$ is not multiplied by a factor of $\frac{1}{\delta}$. Please see Table \ref{tab:3} for details.  

\subsection{Sub-Gaussian Loss Function}
In some contexts, such as regression, considering bounded loss functions is restrictive. Next, we relax the restrictions on the loss function to deal with unbounded losses. We assume that, for any $h \in \mathcal{H}$,  $\ell(h(x), y)$ is \textit{sub-Gaussian}.
First, we mention the definition of sub-Gaussian random variable \cite{wainwright_2019}.
\begin{definition} 
	\label{def_subgaussian}
	A random variable $Z$ is said to be sub-Gaussian with the expectation $\E[Z]=\mu$ and variance proxy $\sigma^2$ if for any $\lambda \in \R$,
	\begin{equation*}
	\begin{split}
	\E[e^{\lambda(Z-\mu)}] \leq e^{\frac{\lambda^2 \sigma^2}{2}}
	\end{split}
	\end{equation*}
\end{definition}

Next, we present our PAC-Bayesian bounds.
\begin{proposition}[The PAC-Bayesian bounds for sub-Gaussian loss function] 
	\label{prop_pac_subgauss}
	Let $P$ be a fixed prior distribution over an infinite hypothesis space $\mathcal{H}$. For a given posterior distribution $Q$ over an infinite hypothesis space $\mathcal{H}$, let $R_D(G_Q)$ and $R_S(G_Q)$ be the Gibbs risk and the empirical Gibbs risk as in Equation \eqref{gibbsrisk} and \eqref{gibbsempiricalrisk} respectively. 
	For the sample size $m$ and $\alpha > 1$, with probability at least $1-\delta$, simultaneously for all posterior distributions $Q$, we have
	\begin{equation*}
	\begin{split}
	R_D(G_Q)  \leq R_S(G_Q)+\sqrt{\frac{2\sigma^2}{m}\log(\frac{2}{\delta})\big(\alpha(\alpha-1)D_{\alpha}(Q\|P)+1\big)^{\frac{1}{\alpha}}}
	\end{split}
	\end{equation*}
	\begin{equation*}
	\begin{split}
	R_D(G_Q)  \leq R_S(G_Q)+\sqrt{\frac{1}{m}\bigg(D_{\alpha}(Q\|P)+\frac{1}{\alpha(\alpha-1)}\bigg)+\frac{1}{m\alpha}\bigg(2\sigma^2(\alpha-1)\log(\frac{2}{\delta})\bigg)^{\frac{\alpha}{\alpha-1}}}
	\end{split}
	\end{equation*}
\end{proposition} 
\begin{proof}
	Suppose the  convex function $\Delta: \R\times\R \rightarrow \R$ is defined as in Proposition \ref{prop_boundloss}. Employing Chernoff's bound, the tail bound probability for sub-Gaussian random variables \cite{wainwright_2019} is given as follows
	\begin{equation}
	\label{tailboundprob_subgaussian}
	\begin{split}
	Pr(|\Bar{Z}-\mu|\geq \epsilon) \leq 2e^{-\frac{m\epsilon^2}{2\sigma^2}}
	\end{split}
	\end{equation}
	Setting $\overline{Z}=R_S(h)$, $\mu=R_D(h)$ and $\delta = 2e^{-\frac{m\epsilon^2}{2\sigma^2}}$ in the tail bound in Equation \eqref{tailboundprob_subgaussian}, for any $h \in \mathcal{H}$, we have
	\begin{equation*}
	\begin{split}
	&\mathop{Pr}_{(x,y) \sim D}(\phi_D(h)\geq\epsilon) \\ 
	&=\mathop{Pr}_{(x,y) \sim D}\bigg((R_S(h)-R_D(h))^2\geq\frac{\epsilon}{t}\bigg) \\
	&=\mathop{Pr}_{(x,y) \sim D}\bigg(\bigg|\frac{1}{m} \sum_{i=1}^{m} \ell(h(x_i), y_i)-\mathop{\E}_{(x,y) \sim D} \ell(h(x), y)\bigg|\geq\sqrt{\frac{\epsilon}{t}}\bigg) \\
	& \leq 2e^{-\frac{m\epsilon}{2\sigma^2t}}
	\end{split}
	\end{equation*}
	
	Setting $\delta = 2e^{-\frac{m\epsilon}{2\sigma^2t}}$, we have
	\begin{equation*}
	\phi_D(h) \leqdel \frac{2t\sigma^2}{m}\log(\frac{2}{\delta})
	\end{equation*}
	where $\phi_D(h)$ is defined as in Equation \eqref{jensenrisk}. By Equation \eqref{momentbound}, we have 
	\begin{equation}
	\label{momentbound_subgauss_1}
	\mathop{\E}_{h\sim P}[(\phi_D(h))^{\frac{\alpha}{\alpha-1}}] \leqdel \bigg(\frac{2t\sigma^2}{m}\log(\frac{2}{\delta})\bigg)^{\frac{\alpha}{\alpha-1}}
	\end{equation}
	On the other hand, we may upper-bound $\mathop{\E}_{h\sim P}[(\phi_D(h))^{\frac{\alpha}{\alpha-1}}]$ in the following way (as in Proposition 6. in \cite{Alquier2018}).
	\begin{equation*}
	\begin{split}
	\mathop{\E}_{h\sim P}[(\phi_D(h))^{\frac{\alpha}{\alpha-1}}] 
	&\leqdel  \frac{t^{\frac{\alpha}{\alpha-1}}}{\delta}\mathop{\E}_{(x,y)\sim D}\mathop{\E}_{h\sim P}[(R_S(h)-R_D(h))^{\frac{2\alpha}{\alpha-1}}]\\
	& = \frac{t^{\frac{\alpha}{\alpha-1}}}{\delta}\mathop{\E}_{h\sim P}\mathop{\E}_{(x,y)\sim D}[(R_S(h)-R_D(h))^{\frac{2\alpha}{\alpha-1}}]\\
	\end{split}
	\end{equation*}
	Let $U=R_S(h)-R_D(h)$ and $q=\frac{\alpha}{\alpha-1}$. Then, $U$ is a sub-Gaussian random variable from our assumption.
	
	\begin{equation*}
	\begin{split}
	&\mathop{\E}_{(x,y)\sim D}[(R_S(h)-R_D(h))^{\frac{2\alpha}{\alpha-1}}] = \mathop{\E}_{(x,y)\sim D}[U^{2q}]\\
	&= \int_0^{\infty}Pr\{|U|^{2q}>u\}du = 2q\int_0^{\infty}u^{2q-1}Pr\{|U|>u\}du\\
	&\leq 4q\int_0^{\infty}u^{2q-1}e^{-\frac{mu^2}{2\sigma^2}}du
	\end{split}
	\end{equation*}
	
	By setting $u=\sqrt{t}$, the previous inequality becomes
	\begin{equation*}
	\begin{split}
	&\mathop{\E}_{(x,y)\sim D}[(R_S(h)-R_D(h))^{\frac{2\alpha}{\alpha-1}}]  \leq 2q\int_0^{\infty}t^{q-1}e^{-\frac{mt}{2\sigma^2}}dt\\
	&=2q\Gamma(q)\bigg(\frac{m}{2\sigma^2}\bigg)^{-q} =2\frac{\alpha}{\alpha-1}\Gamma\bigg(\frac{\alpha}{\alpha-1}\bigg)\bigg(\frac{2\sigma^2}{m}\bigg)^{\frac{\alpha}{\alpha-1}}\\
	\end{split}
	\end{equation*}
	Therefore, we have
	\begin{equation}
	\label{momentbound_subgauss_2}
	\begin{split}
	\mathop{\E}_{h\sim P}[(\phi_D(h))^{\frac{\alpha}{\alpha-1}}]   \leqdel \frac{2\alpha t^{\frac{\alpha}{\alpha-1}}}{\delta(\alpha-1)}\Gamma\bigg(\frac{\alpha}{\alpha-1}\bigg)\bigg(\frac{2\sigma^2}{m}\bigg)^{\frac{\alpha}{\alpha-1}}
	\end{split}
	\end{equation}
	Although one has choices to use either Equation \eqref{momentbound_subgauss_1} or \eqref{momentbound_subgauss_2}, we can easily see that Equation \eqref{momentbound_subgauss_1} is  always tighter than  \eqref{momentbound_subgauss_2}. 
	Putting everything together with Lemma \ref{lemma_cmi_unconstrainedalpha}, we prove our claim.
	\qedhere
\end{proof}

The following corollary is an immediate consequnece of Proposition \ref{prop_pac_subgauss}  for $\alpha=2$.

\begin{corollary}[The PAC-Bayesian bounds with $\chi^2$-divergence for sub-Gaussian loss function]
	\label{cor_pac_subgauss_chi2}
	Let $P$ be any prior distribution over an infinite hypothesis space $\mathcal{H}$. For a given posterior distribution $Q$ over an infinite hypothesis space $\mathcal{H}$, let $R_D(G_Q)$ and $R_S(G_Q)$ be the Gibbs risk and the empirical Gibbs risk as in Equation \eqref{gibbsrisk} and \eqref{gibbsempiricalrisk} respectively. 
	For the sample size $m>0$, with probability at least $1-\delta$, simultaneously for all posterior distributions $Q$, we have	
	\begin{equation}
	\label{cor_pac_subgauss_1}
	\begin{split}
	R_D(G_Q) &  \leq  R_S(G_Q) + \sqrt{\frac{2\sigma^2}{m}\log(\frac{2}{\delta})\sqrt{\chi^2(Q\|P)+1}}
	\end{split}
	\end{equation}	
	\begin{equation}
	\label{eq_pac_subgauss_chi2_2}
	\begin{split}
	R_D(G_Q)  \leq  R_S(G_Q) + \sqrt{\frac{1}{2m}\bigg(\chi^2(Q\|P)+1+\bigg(2\sigma^2\log(\frac{2}{\delta})\bigg)^2\bigg)}
	\end{split}
	\end{equation}
\end{corollary}

\cite{Alquier2018} proved PAC-Bayes bound for sub-Gaussian loss function with $\chi^2$-divergence. It is noteworthy that $\Delta$ may be any convex function and a different choice of $\Delta$ leads us to various bounds. For instance, choosing $\Delta(q,p)=|q-p|$ for Lemma \ref{lemma_mulalpha} results in 
\begin{equation*}
\begin{split}
R_D(G_Q) &  \leq  R_S(G_Q) + \sqrt{\frac{2\sigma^2}{m}\log(\frac{2}{\delta})(\chi^2(Q\|P)+1)}
\end{split}
\end{equation*}	
which is quite similar to the bounds in Proposition 6 in \cite{Alquier2018} and looser than Equation \eqref{cor_pac_subgauss_1}. We obtained a tighter bound due to our choice of $\Delta(q, p)=(q-p)^2$.

	\subsection{Sub-Exponential Loss Function}
	We now turn to a more general class where $\ell(h(x), y)$ is sub-exponential for any $h \in \mathcal{H}$. First, we define sub-exponentiality \cite{wainwright_2019}.
	\begin{definition}
		\label{def_subexp}
		A random variable $Z$ is said to be sub-exponential with the expectation $\E[Z]=\mu$ and parameters $\sigma^2$ and $\beta>0$, if for any $\lambda \in \R$,
		\begin{equation*}
		\begin{split}
		\E[e^{\lambda(Z-\mu)}] \leq e^{\frac{\lambda^2 \sigma^2}{2}}, \forall:|\lambda|<\frac{1}{\beta}
		\end{split}
		\end{equation*}
	\end{definition}
	Next, we provide our PAC-Bayesian bounds.
	\begin{proposition} [The PAC-Bayesian bounds for sub-exponential loss function]
		\label{prop_pac_subexp}
		Let $P$ be a fixed prior distribution over an infinite hypothesis space $\mathcal{H}$. For a given posterior distribution $Q$ over an infinite hypothesis space $\mathcal{H}$, let $R_D(G_Q)$ and $R_S(G_Q)$ be the Gibbs risk and the empirical Gibbs risk as in Equation \eqref{gibbsrisk} and \eqref{gibbsempiricalrisk} respectively. 
		For the sample size $m$ and $\alpha > 1$, with probability at least $1-\delta$, simultaneously for all posterior distributions $Q$, we have
		\begin{equation*}
		\begin{split}
		R_D(G_Q)  \leq R_S(G_Q)+\sqrt{\mathcal{K}^{1}_{\delta}\big(\alpha(\alpha-1)D_{\alpha}(Q\|P)+1\big)^{\frac{1}{\alpha}}}
		\end{split}
		\end{equation*}
		\begin{equation*}
		\begin{split}
		R_D(G_Q)  \leq R_S(G_Q)+\sqrt{\frac{1}{t}\bigg(D_{\alpha}(Q\|P)+\frac{1}{\alpha(\alpha-1)}\bigg)+\mathcal{K}^2_{\alpha, \delta}}
		\end{split}
		\end{equation*}
		where 
		\begin{equation*}
		\begin{split}
		\mathcal{K}^{1}_{\delta}=
		\begin{cases} 
		\frac{2\sigma^2}{m}\log(\frac{2}{\delta}) ,  \quad \frac{2\beta^2\log(\frac{2}{\delta})}{\sigma^2} \leq m\\ 
		(\frac{2\beta}{m}\log \frac{2}{\delta})^2 ,  \quad 0<m<\frac{2\beta^2\log(\frac{2}{\delta})}{\sigma^2},
		\end{cases} 
		\mathcal{K}^{2}_{\alpha, \delta}=\frac{m^{\frac{1}{\alpha-1}}(\alpha-1)^{\frac{\alpha}{\alpha-1}}}{\alpha}(\mathcal{K}^{1}_{\delta})^{\frac{\alpha}{\alpha-1}}
		\end{split}
		\end{equation*}
	\end{proposition} 
	\begin{proof}
		Suppose the convex function $\Delta: \R\times\R \rightarrow \R$ is defined as in Proposition \ref{prop_boundloss}.  For any random variables satisfying Definition \ref{def_subexp}, we have the following concentration inequality \cite{wainwright_2019}.
		\begin{equation}
		\label{tailboundprob_subexp}
		\begin{split}
		Pr(|\overline{Z}-\mu|\geq \epsilon) \leq \begin{cases} 
		2e^{-\frac{m\epsilon^2}{2\sigma^2}},  \quad 0<\epsilon\leq \frac{\sigma^2}{\beta}\\
		2e^{-\frac{m\epsilon}{2\beta}}, \quad \frac{\sigma^2}{\beta}<\epsilon
		\end{cases}
		\end{split}
		\end{equation}
		Following the proof of Proposition \ref{prop_pac_subgauss}, for $0<\sqrt{\frac{\epsilon}{t}} \leq \frac{\sigma^2}{\beta}$, we have 
		\begin{equation*}
		\phi_D(h) \leqdel \frac{2t\sigma^2}{m}\log(\frac{2}{\delta})
		\end{equation*}
		
		For $\frac{\sigma^2}{\beta}<\sqrt{\frac{\epsilon}{t}}$, we have
		\begin{equation*}
		\begin{split}
		&\mathop{Pr}_{(x,y) \sim D}(\phi_D(h)\geq\epsilon) \\ 
		&=\mathop{Pr}_{(x,y) \sim D}\bigg((R_S(h)-R_D(h))^2\geq\frac{\epsilon}{t}\bigg) \\
		&=\mathop{Pr}_{(x,y) \sim D}\bigg(\bigg|\frac{1}{m} \sum_{i=1}^{m} \ell(h(x_i), y_i)-\mathop{\E}_{(x,y) \sim D} \ell(h(x), y)\bigg|\geq\sqrt{\frac{\epsilon}{t}}\bigg) \\
		& \leq 2e^{-\frac{m}{2\beta}\sqrt{\frac{\epsilon}{t}}}
		\end{split}
		\end{equation*}
		Setting $\delta = 2e^{-\frac{m}{2\beta}\sqrt{\frac{\epsilon}{t}}}$, we have, for $\frac{\sigma^2}{\beta}<\sqrt{\frac{\epsilon}{t}}$,
		\begin{equation*}
		\phi_D(h) \leqdel t\bigg(\frac{2\beta}{m}\log \frac{2}{\delta}\bigg)^2
		\end{equation*}
		Thus,
		\begin{equation*}
		\begin{split}
		\phi_D(h) \leqdel 
		\begin{cases}
		\frac{2t\sigma^2}{m}\log(\frac{2}{\delta}) ,  \quad \frac{2\beta^2\log(\frac{2}{\delta})}{\sigma^2} \leq m\\
		t(\frac{2\beta}{m}\log \frac{2}{\delta})^2,  \quad 0<m<\frac{2\beta^2\log(\frac{2}{\delta})}{\sigma^2}
		\end{cases}
		\end{split}
		\end{equation*}
		where $\phi_D(h)$ is defined as in Equation \eqref{jensenrisk}. 
		Now, we have the upper bound for $\phi_D(h)$ so we can apply the same procedure as in Proposition \ref{prop_boundloss} and \ref{prop_pac_subexp}.
		\qedhere
	\end{proof}
	
	Note that, for $\frac{2\beta^2\log(\frac{2}{\delta})}{\sigma^2} \leq m$, $\phi_D(h)$ behaves like sub-Gaussian. However, when the sample size is small, a tighter bound (i.e., $t(\frac{2\beta}{m}\log \frac{2}{\delta})^2$) can be obtained. This shows the advantage  of  assuming  sub-exponentiality  over sub-Gaussianity. 	By setting $\alpha=2$ in Proposition \ref{prop_pac_subexp}, we have the following corollary.
	
	\begin{corollary}[The PAC-Bayesian bounds with $\chi^2$-divergence for sub-exponential loss function]
		\label{cor_pac_sub_exp_chi}
		Let $P$ be any prior distribution over an infinite hypothesis space $\mathcal{H}$. For a given posterior distribution $Q$ over an infinite hypothesis space $\mathcal{H}$, let $R_D(G_Q)$ and $R_S(G_Q)$ be the Gibbs risk and the empirical Gibbs risk as in Equation \eqref{gibbsrisk} and \eqref{gibbsempiricalrisk} respectively. 
		For the sample size $m>0$, with probability at least $1-\delta$, simultaneously for all posterior distributions $Q$, we have	
		\begin{equation*}
		\begin{split}
		R_D(G_Q) &  \leq  R_S(G_Q) + \sqrt{\mathcal{K}^{1}_{\delta}\sqrt{\chi^2(Q\|P)+1}}
		\end{split}
		\end{equation*}	
		\begin{equation*}
		\begin{split}
		R_D(G_Q) &  \leq  R_S(G_Q) + \sqrt{\frac{1}{2m}\big(\chi^2(Q\|P)+1+(m\mathcal{K}^{1}_{\delta})^2\big)}
		\end{split}
		\end{equation*}
		where 
		\begin{equation*}
		\begin{split}
		&\mathcal{K}^{1}_{\delta}=
		\begin{cases} 
		\frac{2\sigma^2}{m}\log(\frac{2}{\delta}) ,  \quad \frac{2\beta^2\log(\frac{2}{\delta})}{\sigma^2} \leq m\\ 
		(\frac{2\beta}{m}\log \frac{2}{\delta})^2 ,  \quad 0<m<\frac{2\beta^2\log(\frac{2}{\delta})}{\sigma^2}
		\end{cases} 
		\end{split}
		\end{equation*}
	\end{corollary}
	
	To the best of our knowledge, our bounds for sub-exponential losses are entirely novel.
	
\section{Log-concave distribution}
	\label{appendix_logconcave}
	We say that a distribution $P$ with a density $p$ (with respect to the Lebesgue measure) is a strongly log-concave distribution if the function $\log p$ is strongly concave. Equivalently stated, this condition means that the density can be expressed as $p(x)=\exp(-\psi(x))$, where the function $\psi:\R^d \rightarrow \R$ is strongly convex, meaning that there is some $\gamma>0$ such that
	\begin{equation*}
	\lambda \psi(x) + (1-\lambda)\psi(y) -\psi(\lambda x +(1-\lambda)y) \geq \frac{\gamma}{2} \lambda(1-\lambda) \|x-y\|_2^2
	\end{equation*}
	for all $\lambda \in [0,1]$, and $x,y \in \R^d$.

\end{document}